\RequirePackage{amsthm}
\documentclass[referee,sn-basic]{sn-jnl}

\usepackage{graphicx}%
\usepackage{amsmath,amssymb,amsfonts}%
\usepackage{amsthm}%
\usepackage{mathrsfs}%
\usepackage[title]{appendix}%
\usepackage{xcolor}%
\usepackage{textcomp}%
\usepackage{manyfoot}%
\usepackage{booktabs}%
\usepackage{listings}%
\usepackage[nameinlink]{cleveref}

\crefname{algocf}{alg.}{algs.}
\Crefname{algocf}{Algorithm}{Algorithms}

\usepackage{mathtools}

\usepackage{csquotes}
\usepackage{subcaption}
\usepackage{mathtools,amsmath}
\usepackage[ruled,vlined]{algorithm2e}
\newcommand{\R}{\mathbb{R}}
\newcommand{\E}{\mathbb{E}}

\newcommand{\argmax}{\operatorname{arg\,max}}
\newcommand{\argmin}{\operatorname{arg\,min}}
\newcommand{\bld}[1]{\boldsymbol{#1}}
\newcommand{\btheta}{\boldsymbol{\theta}}
\newcommand{\bpi}{\boldsymbol{\pi}}
\newcommand{\bTheta}{\boldsymbol{\Theta}}
\newcommand{\SPE}{\operatorname{SPE}}

\newenvironment{tightlist}
{\begin{list}{$\bullet$}{
    \setlength{\topsep}{0in}
    \setlength{\partopsep}{0in}
    \setlength{\itemsep}{0in}
    \setlength{\parsep}{0in}
    \setlength{\leftmargin}{1.5em}
    \setlength{\rightmargin}{0in}
    \setlength{\itemindent}{0in}
}
}%
{\end{list}
}
\usepackage{sgame}

\theoremstyle{thmstyleone}%
\newtheorem{theorem}{Theorem}%  meant for continuous numbers

\newtheorem{proposition}[theorem]{Proposition}% 
\newtheorem{corollary}[theorem]{Corollary}
\theoremstyle{thmstyletwo}%
\newtheorem{example}{Example}%

\theoremstyle{thmstylethree}%
\newtheorem{definition}{Definition}%

\begin{document}

\title[Formal Contracts Mitigate Social Dilemmas in Multi-Agent RL]{Formal Contracts Mitigate Social Dilemmas in Multi-Agent Reinforcement Learning}

\author*[1]{\fnm{Andreas} \sur{Haupt}}\email{haupt@mit.edu}
\equalcont{These authors contributed equally to this work.}
\author*[1]{\fnm{Phillip} \sur{Christoffersen}}\email{philljkc@mit.edu}
\equalcont{These authors contributed equally to this work.}
\author[1]{\fnm{Mehul} \sur{Damani}}\email{mehul42@mit.edu}
\author[1]{\fnm{Dylan} \sur{Hadfield-Menell}}\email{dhm@csail.mit.edu}

\affil[1]{\orgdiv{Computer Science and Artificial Intelligence Laboratory}, \orgname{Massachusetts Institute of Technology}, \orgaddress{\street{32 Vassar Street}, \city{Cambridge}, \postcode{02139}, \state{MA}, \country{U.S.A}}}

\abstract{Multi-agent Reinforcement Learning (MARL) is a powerful tool for training autonomous agents acting independently in a common environment. However, it can lead to sub-optimal behavior when individual incentives and group incentives diverge. Humans are remarkably capable at solving these social dilemmas. It is an open problem in MARL to replicate such cooperative behaviors in selfish agents. In this work, we draw upon the idea of formal contracting from economics to overcome diverging incentives between agents in MARL. We propose an augmentation to a Markov game where agents voluntarily agree to binding transfers of reward, under pre-specified conditions. Our contributions are theoretical and empirical. First, we show that this augmentation makes all subgame-perfect equilibria of all Fully Observable Markov Games exhibit socially optimal behavior, given a sufficiently rich space of contracts. Next, we show that for general contract spaces, and even under partial observability, richer contract spaces lead to higher welfare. Hence, contract space design solves an exploration-exploitation tradeoff, sidestepping incentive issues. We complement our theoretical analysis with experiments.  Issues of exploration in the contracting augmentation are mitigated using a training methodology inspired by multi-objective reinforcement learning: Multi-Objective Contract Augmentation Learning (MOCA). We test our methodology in static, single-move games, as well as dynamic domains that simulate traffic, pollution management and common pool resource management.}


\maketitle

\section{Introduction}\label{sec:intro}
\begin{figure}[t!]
\centering
\includegraphics[width=\linewidth]{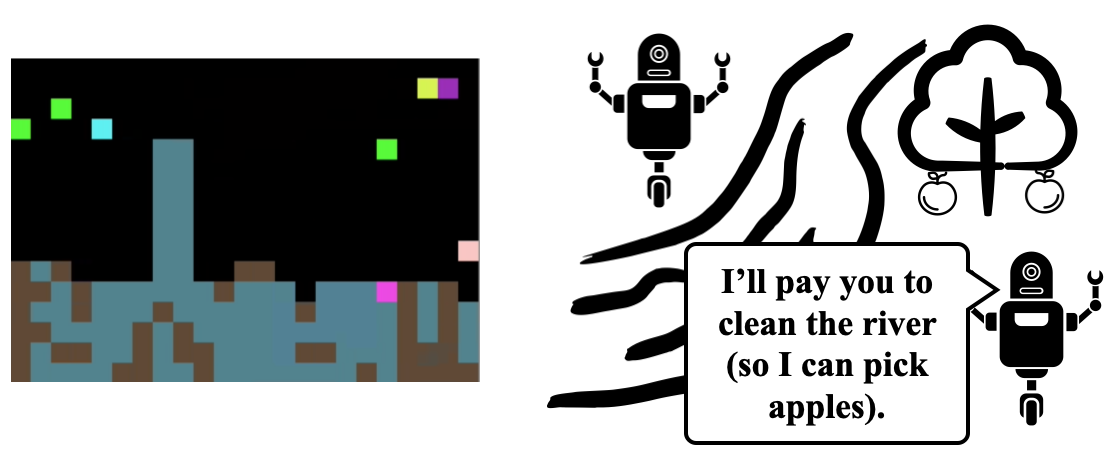}
\caption{We evaluate our method in the Cleanup domain~\citep{hughes2018inequity}. \textbf{Left}: A screenshot of the environment. The different agents correspond to the pink, yellow, and purple tiles. Agents are rewarded for picking apples (green), but apples will only grow if the river (blue) is clean of pollution (brown). Agents can clean up pollution, but aren't directly rewarded for cleaning. This creates a \emph{social dilemma} where no agents clean because they don't expect to benefit from cleaning directly.  \textbf{Right:} An illustration of the solution that our contracting augmentation facilitates. In the Cleanup domain, one agent commits to \enquote{pay} the other to clean the river. As a result, the agents are able to coordinate on policies that maximize the total reward across both agents. 
}
\label{fig:problem}
\end{figure}
We study the problem of how to get selfishly motivated agents to act pro-socially through the lens of Multi-Agent Reinforcement Learning (MARL). 
Consider the Cleanup domain, depicted in \Cref{fig:problem}. Agents get reward for picking apples that only grow if a nearby river is unpolluted. In a pro-social solution to Cleanup, agents need to work together: one cleans while the other eats apples. However, self-interested agents cannot sustain this solution. Cleaning has no direct benefit, so selfish agents focus exclusively on picking apples. A social dilemma ensues.

This article studies contracts as voluntary commitments to zero-sum modifications of the environment reward, so-called \emph{social contracts}, as a detail-free method of mitigating the consequences of conflicting algorithmic incentives. More specifically, contracts transfer rewards between agents depending on \emph{contractible observations}. Contracts are proposed by agents, and can be vetoed by any single agent---participation is voluntary. Even upon acceptance, agents may choose any action, only their rewards are changed by the contract. We show, both game-theoretically and empirically, that the possibility for agents to contract allows for optimal (in a sense we specify) system performance for any initial reward signals to agents. 

Consider again the problem of designing agents to harvest in an orchard where agents are rewarded for picking apples only. In a contract, an agent could propose that an agent receives $r_{\text{clean}}$ of reward for each polluted river square that is cleaned. All other agents are \enquote{charged} a reward penalty of $-\frac{r_{\text{clean}}}{n-1}$. The proposing agent, as well as all non-harvesting agents have an incentive to propose this contract because the expected reward from picking apples will be larger than the expected payment to others for growing them. Similarly, the other agent prefers this contract to the outcome absent a contract, where no apples grow.

As a simpler example, consider the effects of contracts in the classic Prisoner's Dilemma \cite{tucker1983mathematics}. The tables in \Cref{tab:pd} show the payoffs for the unmodified game (\Cref{subfig:pd}) and the modified incentives (\Cref{subfig:finpd}) under the following contract: 
\begin{quote}
\emph{Any agent who defects is fined $1.5$ units of reward by the other agent.}
\end{quote}
If both defect, both pay, and the payments cancel. In this modified game, cooperation is dominant, and hence $(C,C)$ is the only Nash equilibrium.
\begin{figure}[t!]
\begin{subfigure}[b]{0.45\linewidth}
\centering
\begin{tabular}{ccc}
\toprule
& $C$ & $D$ \\
\midrule
$C$ & $-1$,$-1$ & $-3$,$0$ \\
\midrule
$D$ & $0,-3$ & $-2,-2$ \\
\bottomrule
\end{tabular}
\caption{Prisoner's Dilemma}
\label{subfig:pd}
\end{subfigure}
\begin{subfigure}[b]{0.45\linewidth}
\centering
\begin{tabular}{ccc}
\toprule
& $C$ & $D$ \\
\midrule
$C$ & $-1$,$-1$ & $-1.5$,$-1.5$ \\
\midrule
$D$ & $-1.5,-1.5$ & $-2,-2$ \\
\bottomrule
\end{tabular}
\caption{After Contract}
\label{subfig:finpd}
\end{subfigure}
\caption{
(a) Prisoner's Dilemma (b) Prisoner's Dilemma after signing a contract in which a defector transfers $1.5$ reward to a cooperator. With this contract in force, cooperating becomes a dominant action for both players.}
\label{tab:pd}
\end{figure}

Would the agents agree to this contract if proposed? If it is rejected, they subsequently play the game in \Cref{subfig:pd} which has a unique Nash equilibrium of $(D, D)$, yielding a reward of $-2$ for both agents. However, if it is accepted, they subsequently play the game in \Cref{subfig:finpd} which has a Nash equilibrium of $(C, C)$, yielding a reward of $-1$ for both agents. Thus, agents want to accept the contract and subsequently play the socially optimal outcome. Hence, a possibility to commit to a state-action dependent reward transfer, or enter into a formal contract, mitigates a social dilemma, even among selfish agents.

In this article, we provide four main contributions. 
\begin{enumerate}
\item We formalize \emph{Formal Contracting} as a generic augmentation of Markov games (i.e., non-cooperative MARL);
\item We prove that this augmentation makes socially optimal behavior a \emph{subgame-perfect equilibrium}, and that every SPE of the augmented game is socially optimal, if sufficiently complex contracts are allowed and the game is sufficiently observable in a sense we specify;
\item We prove in more general settings that more expressive contracts allow for higher system performance in terms of the sum of rewards;
\item We implement the contracting augmentation using state-of-the-art deep reinforcement learning. Contracting improves system performance in several complex game environments. We propose an algorithm to overcome exploration challenges, which performs close to, or better than, a joint controller after a fixed number of time periods in complex dynamic domains such as Cleanup and Harvest, see \cite{hughes2018inequity}. 
\end{enumerate}
The rest of this article is structured as follows. We provide preliminaries and define our augmentation in \Cref{sec:background}. In \Cref{sec:theory}, we show that formal contracting mitigates social dilemmas as long as contracts can detect deviations from a socially optimal policy profile. In \Cref{sec:features} we formalize the notion of contractible features, and prove that the system performance improves with richer contracting features. We describe our evaluation methodology and introduce a learning algorithm for subgame perfection, multi-objective contract augmentation learning, in \Cref{sec:methodology}. Experimental results are presented in \Cref{sec:experiments}. We review related work in \Cref{sec:related}. In \Cref{sec:conclusion}, we discuss the real-world application and enforcement of contracts, fairness concerns, and avenues for future work. Appendices contain proofs, additional statements and experiments, a formal definition of a more general contracting augmentation, and hyper-parameter settings for our experiments.

\section{Formal Contracting}\label{sec:background}
\subsection{Background}
\subsubsection{Fully Observable Markov Games} We define an $N$-agent Fully Observable Markov Game (FOMG) as a 6-tuple, $M = \langle S,s_0, \mathbf A, T, \mathbf R, \gamma \rangle$, where
\begin{itemize}
\item $S$ is a state space;
\item $s_0 \in S$ is the initial state;
\item $\mathbf A = A_1 \times A_2 \times \dots \times A_n$ is the space of action profiles $\mathbf a = (a_1, a_2, \dots, a_n)$ for $N$ agents;
\item $T \colon S \times \mathbf A \to \Delta(S)$ is a transition function;
\item $\mathbf R \colon S \times \mathbf A \to [-R_{\max}, R_{\max}]^n$ is a (bounded) reward function mapping state-action profiles to reward vectors for the $n$ agents; and
\item $\gamma \in [0,1)$ is a discount factor.
\end{itemize}
We will use bold symbols to denote sets of tuples. Agents choose policies $\pi_i \colon S \to \Delta(A_i)$, $i=1, 2, \dots, n$. We write $\boldsymbol\pi \coloneqq (\pi_1, \pi_2, \dots, \pi_N)$ to denote a \emph{policy profile}. For a policy profile $\boldsymbol\pi$, we denote by 
\begin{equation}
V_i^{\boldsymbol\pi}(s_0) \coloneqq \E\left[\sum_{t=0}^\infty \gamma^t R_i(s_t, \mathbf a_t)\right]\label{eq:value}
\end{equation}
the \emph{value} to agent $i \in [N] \coloneqq \{ 1, 2, \dots, N\}$. The expectation in \eqref{eq:value} is on the process $s_t \sim T(s_{t-1}, \mathbf a_{t-1})$ and $a_{t,i} \sim \pi_i(s_t)$, $t=1, 2, \dots$.  A subscript ${()}_{-i}$ denotes a partial profile of policies or actions or policies excluding agent $i$, e.g., $\boldsymbol\pi_{-i} \coloneqq (\pi_1, \pi_2, \dots, \pi_{i-1}, \pi_{i+1}, \dots, \pi_N)$.

\subsubsection{Partially Observable Markov Games} An $N$-agent Partially Observable Markov Game (POMG) is an 8-tuple 
\[
P = \langle S, s_0, \Omega, \mathbf A, T, \mathbf O, \mathbf R, \gamma \rangle.
\]
$S, s_0, A$ and $T$ are defined as in Fully Observable Markov games. $\Omega$ is an observation space, and $\mathbf{O}: S \times \mathbf A \to (\Delta(\Omega)^n$ an observation model, mapping world states to a tuple of observations. In POMGs, agents' policies condition on sequences of observations and actions, $\pi: H \to \Delta(A_i)$, where $H$ denotes the set of observation-action histories. Denote
\[
V_i^{\boldsymbol\pi}(h) \coloneqq \E\left[\sum_{t=0}^\infty \gamma^t R_i(s_t, \mathbf a_t) \middle\vert h_t\right]
\]
where the expectation is given with respect to the process $h_{t+1} = (h_t: s_t, a_t, o_t)$, where $:$ denotes concatenation, $s_t \sim T(s_{t-1}, \mathbf{a}_{t-1})$ and $a_{t, i} \sim \pi_i(s_t)$. Write $V_i^{\boldsymbol{\pi}}(h_0) \coloneqq V_i^{\boldsymbol{\pi}}(\emptyset)$ for the empty history $\emptyset$. For the choice $\Omega = S$, $O_i (s,\mathbf a) = (s,\mathbf a)$ for all $s \in S$, the class of POMGs recovers FOMGs.

We will consider an additional observation model which we call the \emph{contractible observations}. Let $P$ be a POMG, and $O_0: S \times \mathbf A \to \Delta(\Omega)$. We denote the observations depending on a particular strategy profile $\bpi$ by a superscript, i.e. $o_t^{\bpi}$ is the observation at timepoint $t$ when agents follow the strategy profile $\bpi$. 

\begin{definition}
    Observation model $O_0$ is \textit{sufficient to detect deviators from $\boldsymbol{\pi}$} in POMG $P$ if, for all $i=1, 2, \dots, n$ and $\pi_i^\prime \ne {\pi}_i$, we have that 
    \[
    \mathbb{P}\left[\exists T \in \mathbb{N}: o_{0t}^{(\pi_i', \bpi_{-i})} \neq o_{0t}^{\bpi}\right] = 1.
    \]
\end{definition}
To unpack this equation, an observation model is sufficient to detect deviation if some future observations are able to distinguish $\bpi$ from any policy profile in which a single agent deviates, $(\pi_i, \boldsymbol{\pi}_{-i})$.

\subsubsection{System Performance and Optimality}
This article measures system performance through \emph{welfare},
\[
W^{\boldsymbol\pi} (s_0)\coloneqq \sum_{i=1}^n V_i^{\boldsymbol\pi}(s_0).
\]
We refer to a policy profile that maximizes welfare as \emph{jointly optimal}: $\boldsymbol\pi^* \in \argmax_{\boldsymbol\pi} W^{\boldsymbol\pi} (s_0)$. A policy profile $\boldsymbol\pi$ is \emph{Pareto-optimal} if there is no policy profile $\boldsymbol\pi'$ such that $V^{\boldsymbol\pi}_i (s_0) \le V^{\boldsymbol\pi'}_i (s_0)$ for all $i=1, 2, \dots, N$, with a strict inequality for at least one agent. Intuitively, in such profiles, there are no \enquote{win-wins}: no agent can attain higher reward without at least one other agent losing reward.

\subsubsection{Stable Policy Profiles and Equilibria}
In social dilemmas, social and individual incentives diverge. In our game-theoretic analysis, we use an equilibrium notion to capture outcomes of selfish incentives. One potential solution concept is \emph{Nash equilibrium}. A policy profile is Nash equilibrium if unilateral deviation is suboptimal for all agents. Formally, a policy profile $\boldsymbol\pi$ is a Nash equilibrium if for any agent $i \in [n]$ and any policy $\pi_i': S \to \Delta(A_i)$ (or $\pi_i^\prime: H \to \Delta(A_i)$ in POMGs), $V_i^{\boldsymbol\pi} (s_0) \ge V_i^{(\pi_i', \boldsymbol\pi_{-i})} (s_0)$.

While this solution concept is common, it has its drawbacks. For example, in Cleanup, a policy profile in which agents never clean under some contract even if it is in their best interest, and another agent not proposing it, might be a Nash equilibrium: No agent would benefit unilaterally from changing their behavior. The \enquote{threat} of one of the agents to not clean, however, is non-credible, as, when the contract were active, they would rather clean. Compare \cite[Section 5.5]{osborne1994course} on non-credible threats.

To avoid non-credible threats, we model selfish incentives in MARL with \emph{subgame-perfect equilibria} (SPE). Subgame perfection requires that for any state $s$ or history $h$, there cannot be a profitable deviation to another policy, for any agent. This is stronger than a Nash equilibrium, which only requires this to hold at the initial state $s_0$. 
\begin{definition}[Subgame-Perfect Equilibrium]
A policy profile $\boldsymbol\pi$ is a \emph{subgame-perfect Nash equilibrium} or \emph{subgame-perfect} for an FOMG if for all histories $h \in H$, agents $i=1,2 , \dots, N$ and policies $\pi_i' \colon S \to \Delta (A_i)$, it holds
\[
V_i^{\boldsymbol\pi} (h) \ge V_i^{(\pi_i', \boldsymbol\pi_{-i})} (h).\footnote{Some histories might never be reached. We will assume for this article that such \enquote{off-path} values can be chosen arbitrarily. The research program on \emph{equilibrium refinements} studies how to discipline off-path beliefs, compare \cite[ch. 4]{gibbons1992game}.}
\]
\end{definition}

\subsection{The Contracting Augmentation}\label{sec:augmentation}
\begin{figure}[ht]
\centering
\includegraphics[width=\linewidth]{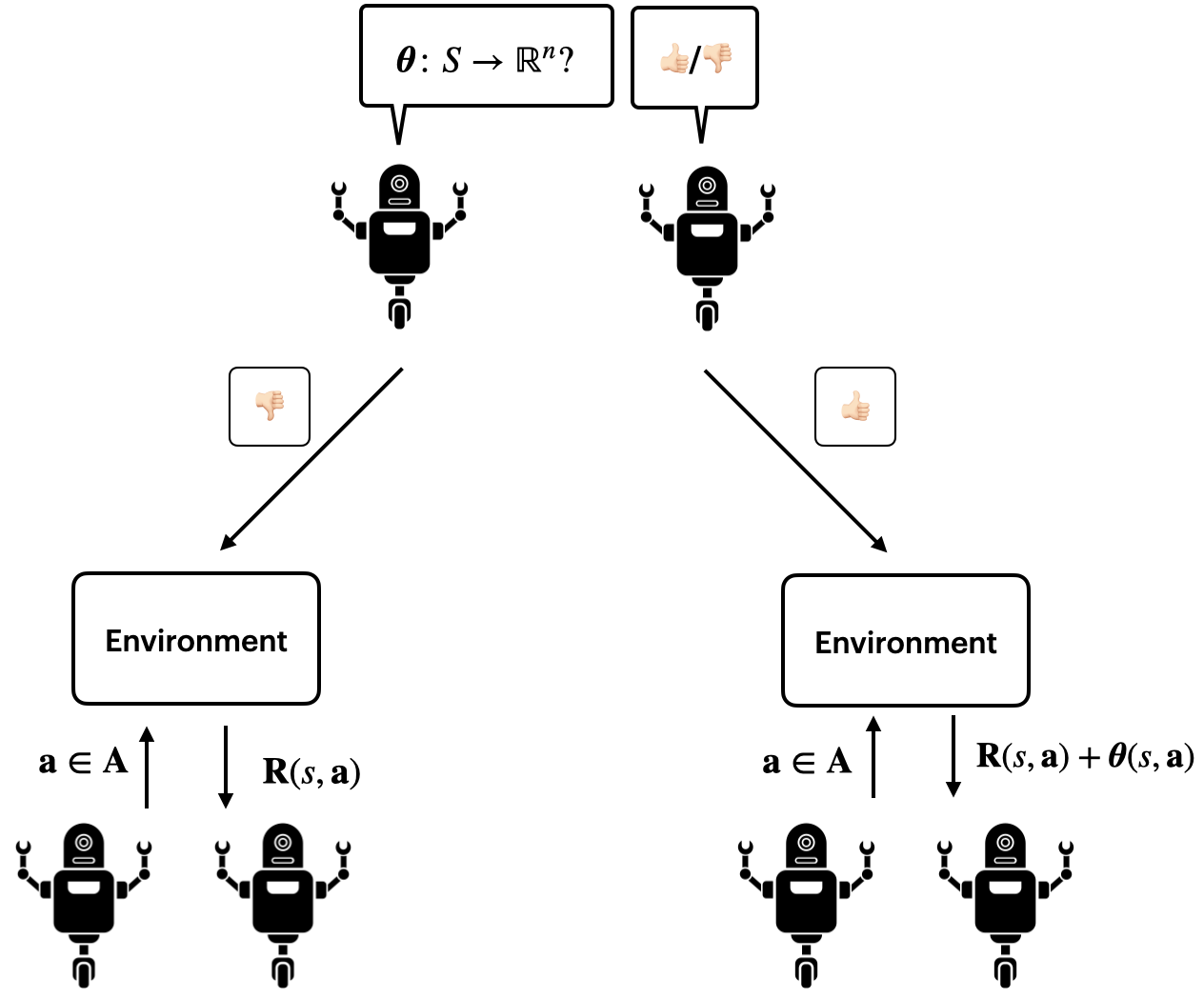}
\caption{The Contracting Augmentation.. \textbf{Top:} Agents can propose \emph{contracts}, state dependent, zero-sum, additive augmentations to their reward functions. Agents can accept or decline contracts. \textbf{Left:} In case of declination, the interaction between agents happens as before. \textbf{Right:} In case of acceptance of the contract, the reward of the agents is altered according to the rules of the contract.}
\label{fig:cover}
\end{figure}
Contractible observation-dependent reward transfers are the main object of study in this article. Here, we formalize these augmentations within Markov games. 
\begin{definition}[Contract]
A contract is a function $\boldsymbol\theta \colon \Omega \to \R^N$ whose range consists of zero-sum vectors, i.e. 
\[
\sum_{i=1}^N \theta_i = 0
\]
for any $(\theta_1, \theta_2, \dots, \theta_n) \in \operatorname{range} (\theta) $. We denote a generic set of contracts by $\boldsymbol\Theta$. 
\end{definition}

Contracts act on the reward that agents get once they are accepted. Moreover, the information used by the contracting logic is modeled by a corresponding contracting observation.

\begin{definition}[Contracting Model]
    A contracting model is a 2-tuple $\mathcal{C} = \langle O_0, \bld{\Theta}\rangle$. $O_0$ is an observation model and $\bld{\Theta}$ is a contract space.
\end{definition}

This allows us to define the contract augmentation.

\begin{definition}[$\boldsymbol{\mathcal{C}}$-Augmented Game]
Let $M = \langle S, s_0, \mathbf A, T, \mathbf R, \gamma \rangle$ be an FOMG and $\mathcal{C} = (O_0, \boldsymbol\Theta)$ be a contracting model. The $i$-proposing, \emph{$\boldsymbol{\mathcal{C}}$-augmented game} is $M^{\boldsymbol{\mathcal{C}}} = \langle S', (i, \boldsymbol 0), \mathbf A', T', \mathbf R', \gamma \rangle$, with the components listed below. If the contractible  observations are clear from context (e.g. when the full state $S$ is provided to the contract), we write $M^\Theta$ instead of $M^\mathcal{C}$.
\end{definition}
We describe the notation first for FOMGs.
\paragraph{States.} The augmented state space is
\[
S' = ([n] \cup S) \times (\{\boldsymbol 0\} \cup \boldsymbol\Theta).
\]
States have the following meanings:
\begin{itemize}
\item $(i, \boldsymbol 0)$: Agent $i$ has the opportunity to propose a contract $\boldsymbol\theta \in \boldsymbol\Theta$;
\item $(i, \boldsymbol\theta)$: $\boldsymbol\theta \in \boldsymbol\Theta$ awaits acceptance or rejection by all agents;
\item $(s, \boldsymbol 0)$: The game is in state $s \in S$ with a null contract, $\boldsymbol 0(s,a) = 0$, for all $s \in S, \mathbf a \in \mathbf A$, in force;
\item $(s, \boldsymbol\theta)$:  The system is in state $s$ with contract $\boldsymbol\theta \in \boldsymbol\Theta$ in force.
\end{itemize}
We will implicitly assume that contractible observation models always perfectly observe the proposed contract, and whether contracts were accepted. 
\paragraph{Actions.}
The action spaces for the agents are
\begin{align*}
A'_i &= A_i \cup \boldsymbol\Theta \cup \{\operatorname{acc}\}
\end{align*}
which corresponds to actions in the game ($A_i$), proposal actions ($\boldsymbol\Theta$) and the acceptance action ($\{\operatorname{acc}\}$). 
\paragraph{Transitions.}
There are deterministic transitions, given by
\begin{align*}
    T'((i, \boldsymbol 0), (\boldsymbol\theta, \mathbf a_{-i})) & = \begin{cases}
    (i, \boldsymbol\theta), & \text{ for any $(\btheta, \mathbf a_{-i})$}\\
    (i, \mathbf{0}) & \text{else.}\end{cases}\\
    T'((i, \boldsymbol\theta), \mathbf a) &=\begin{cases} (s_0, \boldsymbol\theta) & \text{if } \mathbf a = \textbf{acc} \\ (s_0, \boldsymbol 0) & \text{otherwise.} \end{cases}
\end{align*}
for any contract $\boldsymbol\theta \in \boldsymbol\Theta$ and any action profile $\mathbf a \in \mathbf A$. Here, we denoted $\textbf{acc} \coloneqq (\operatorname{acc}, \operatorname{acc}, \dots, \operatorname{acc})$ the profile of unanimous acceptance of a contract.

Transitions in states $(s, \boldsymbol 0)$ and $(s, \boldsymbol\theta)$ are as in the underlying game $M$,
\[
T'((s, \boldsymbol\theta), \mathbf a) = T(s, \mathbf a)
\]
for any $s\in S$, $\boldsymbol\theta \in \boldsymbol\Theta$ and $a \in A$. 
\paragraph{Rewards.}
\begin{equation*}
\begin{split}
\mathbf R'((s, \boldsymbol\theta), \mathbf a) &= \mathbf R(s,\mathbf a) + \boldsymbol\theta(s, a),\\
\mathbf R'((i, \boldsymbol\theta), \textbf{acc}) &= \boldsymbol \theta(\operatorname{acc}).
\end{split}
\end{equation*}
for $\boldsymbol\theta \in \boldsymbol\Theta $ and $s \in S$. In states $(s, \btheta)$, the actions $\operatorname{acc}$ and $\btheta$ yield a high negative reward. The first line means that depending on a state-action profile pair, reward is transferred between the agents. The second line refers to reward being transferred on signing a contract. 

\paragraph{Partially Observable Markov Games} In order to generalize the above construction to POMGs, we consider a set of contracts $\Theta = \{ \Omega \to \R^n \}$ which depend on the contractible observation model $O_0: S \times \mathbf A \to \Delta (\Omega)$, in general distinct from the observation models for each agent. We recover contracting in an FOMG for contractible observations $O_0 (s,\mathbf a ) = (s, \mathbf a)$. The environment dynamics are the same as in the augmentation for FOMGs. It remains, hence, to show the observation and reward model.

The observation space for the agents then becomes $\Omega'= ([n] \cup \Omega) \times (\{\boldsymbol{0}\} \cup \Theta)$. Agents observe the proposer, the contract proposed or in force, and, if they are in a game state, they also observe 
\begin{align*}
    O'_i( (s, \boldsymbol\theta), \mathbf a) &= \begin{cases} (O_i(s, \mathbf a), \boldsymbol\theta) & s \in S  \\ (x, \boldsymbol\theta) & s \in [n]   \end{cases}
\end{align*}
To unpack this notation, observe that $(s, \boldsymbol\theta) \in S'$ is a state in the augmented game, and $\mathbf a$ is an action profile.

In the contracting augmentation, once enforced, the expected rewards of agents are directly changed, according to 
\[
R_i'(s,\mathbf a) = R_i(s,\mathbf a) + \E_{o \sim O_0 (s, \mathbf a)} [\theta_i (o)].
\]
That is, the rewards are sampled according the observation model. As in FOMGs, $R( (i, \btheta), \mathbf{\operatorname{acc}}) = \btheta(\operatorname{acc})$. Note that agents maximize their reward as modified under the contract, so there is no concept of \enquote{breaking} a contract. The incentives that align agents' behavior with pro-social goals are encoded in the reward function. 

\paragraph{Expressiveness.} We define a notion of comparison for a subclass of contracting models. A contracting model $\mathcal{C} = (O_0, \bld{\Theta})$ has greater \emph{expected transfer expressiveness} than $\mathcal{C}^\prime = (O_0', \bld{\Theta'})$, written $\mathcal{C}^\prime \preceq_{\text{ETE}} \mathcal{C}$ if, for all $\bld{\theta^\prime} \in \bld{\Theta^\prime}$, there exists $\bld{\theta} \in \bld{\Theta}$ such that
    \begin{equation}
    \mathbb{E}_{o \sim O_0'(s, \mathbf a)}[\bld{\theta^\prime}(o)] = \mathbb{E}_{o \sim O_0(s, \mathbf a)}[\bld{\theta}(o)]. \label{eq:expressiveness}
    \end{equation}
\section{Formal Contracting Mitigates Social Dilemmas}\label{sec:theory}

In this section, we show that formal contracting, given  a sufficiently rich set of contracts, mitigates social dilemmas in many fully and partially observable settings. Find a summary of our results in the first column of \Cref{fig:table}.

\begin{figure}
    \centering
    \includegraphics[width=1\linewidth]{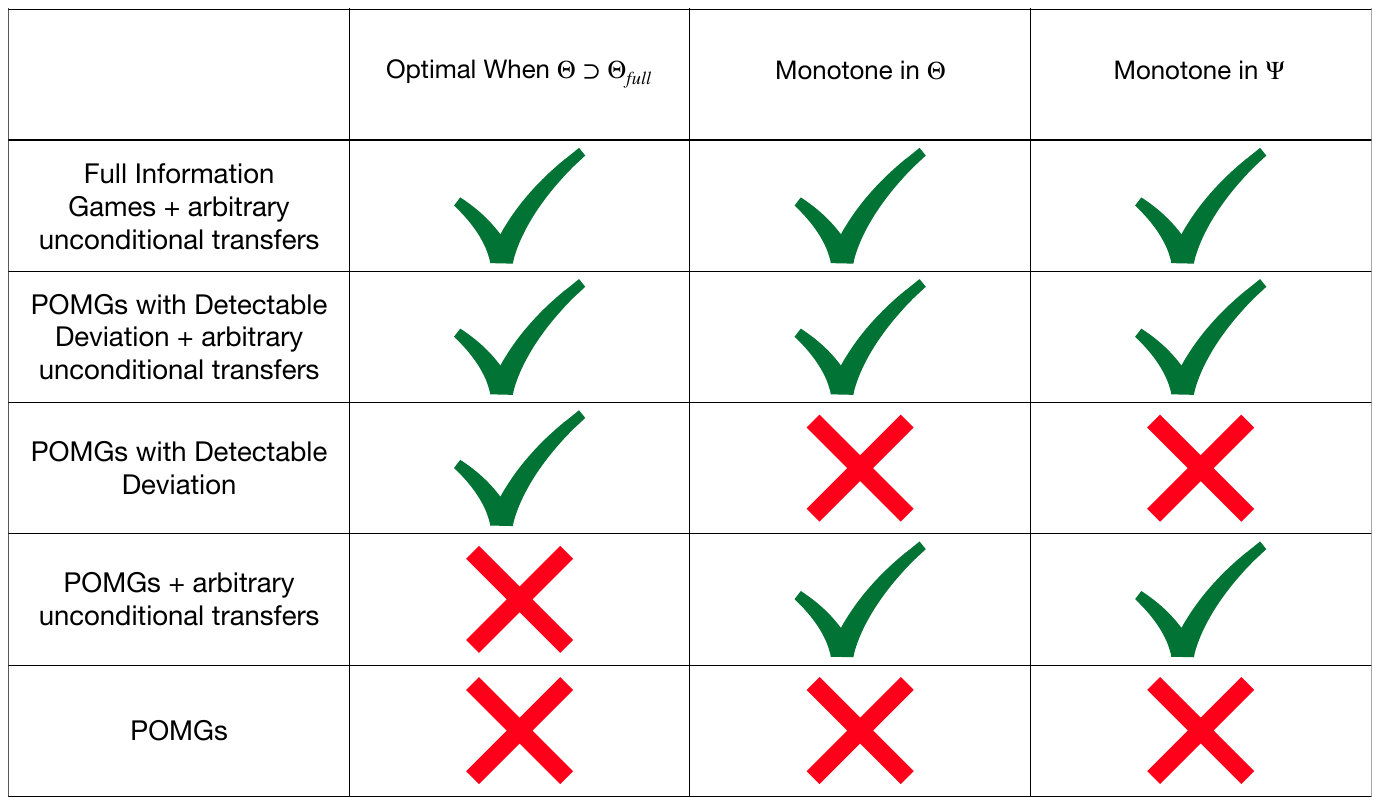}
    \caption{A table summarizing the theoretical results proven in \Cref{sec:theory} and \Cref{sec:features}. The columns along the top list theorem types, which are (1) optimality when you have a contracting space with sufficient richness, (2) monotonicity in the size of the contract space (3) monotonicity in the space of contractible features. The rows list the various problems setups considered in these sections.} 
\label{fig:table}
\end{figure}

\begin{theorem}[Optimality of Contracting]\label{thm:main}
Let $M = \langle S, s_0, \mathbf  A, T, \mathbf R, \gamma \rangle$ be a Fully Observable Markov game. If all observations are contractible, $O_0 (s, \mathbf a) = (s, \mathbf a)$, then the contracting space
\[
\Theta_{\operatorname{full}} = \{(S \times \mathbf A) \cup \{\operatorname{acc}\} \to \R\},
\]
\begin{enumerate}
    \item there exists an SPE of $M^{\mathcal C}$;
    \item for any SPE $\bpi$ of $M^{\mathcal C}$, there is a jointly optimal policy profile $\bpi^*$ of $M$ such that $\bpi ((s, \boldsymbol\theta), \mathbf a) = \boldsymbol\pi^* (s, \mathbf a)$ for the contract $\boldsymbol\theta$ that agent $i$ chooses in $\boldsymbol\pi$.
\end{enumerate}
\end{theorem}
The theorem shows that, under the assumption of richness, contracting mitigates social dilemmas in game-theoretic equilibrium.
\begin{proof}[Proof Sketch.]
    For any socially optimal policy profile $\pi^*$, we can construct a \enquote{forcing contract} $\boldsymbol \theta^*$, which sets a high penalty for not playing $\bld \pi^*$. A signing bonus for this contract transfers all utility gained or lost by playing $\boldsymbol \theta^*$ for all agents but the proposer to the proposer. The main part of the proof is to show that this is the (only) optimal choice for the proposer from $\boldsymbol \Theta_{\operatorname{full}}$, subject to agents accepting the contract.
\end{proof}
The full proof can be found in \Cref{apx:proofs}. The conclusions of this theorem, along with those of \Cref{thm:monotonicity}, are illustrated in \Cref{fig:full_obs}. 
\begin{figure}
    \centering
    \includegraphics[width=1\linewidth]{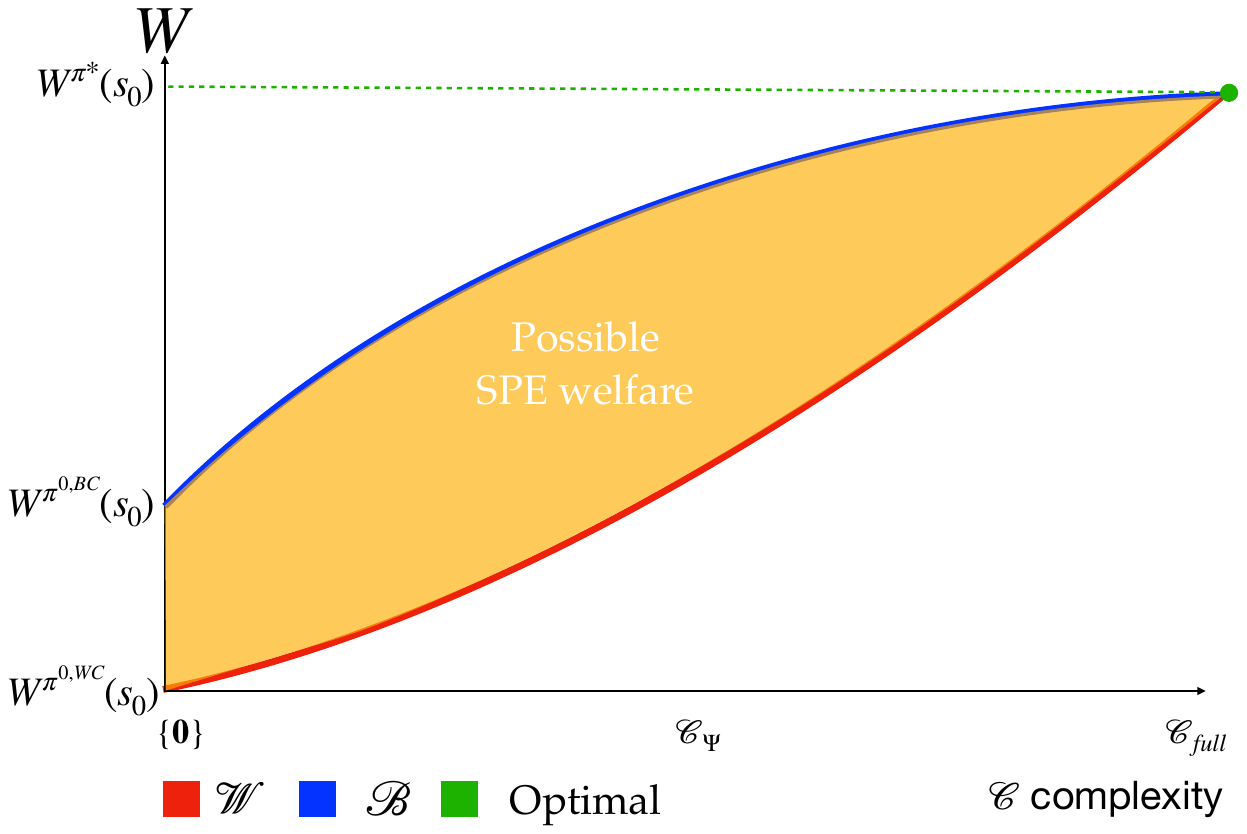}
    \caption{Under detectable deviations, possible welfare levels at equilibrium improve to optimality as the complexity of $\Theta$ grows, if spaces have AUT.} 
\label{fig:full_obs}
\end{figure}
The same result continues to hold for POMGs if deviations are detectable. In this case, it might not be immediately possible  since the policy can simply wait a finite amount of time until a deviation occurs, and the precise agent that committed it is determined, by the detectable  deviation assumption. This is more precisely articulated below. 

\begin{theorem}[Optimality of Contracting with Detectable Deviations]\label{thm:detecatbledev}\label{thm:partial_obs_main}
Let $P = \langle S, s_0, \Omega, \mathbf  A, T, \mathbf{O}, \mathbf R, \gamma \rangle$ be a POMG, and $O_0$ a contractible observation model that has detectable deviations from some socially optimal policy profile $\pi^*$. Then, the same conclusions as in \Cref{thm:main} hold, if we require $\bTheta = \bTheta_{\operatorname{\operatorname{full}}}$ and $\mathcal{C} \succeq_{\operatorname{ETE}} \mathcal{C}_{\operatorname{full}} = (\bTheta_{\operatorname{\operatorname{full}}}, O_0)$.
\end{theorem}
This result, however, fails to hold if the joint observations are unable to detect deviations. One trivial example of this is if contractible observations are perfectly uninformative, $\Omega = \{0\}$. In this case, reward transformations cannot condition on actions played in the course of the game at all, and hence are strategically irrelevant. The augmented game has the same subgame-perfect equilibria as the original game. 
\begin{figure}
    \centering
    \includegraphics[width=1\linewidth]{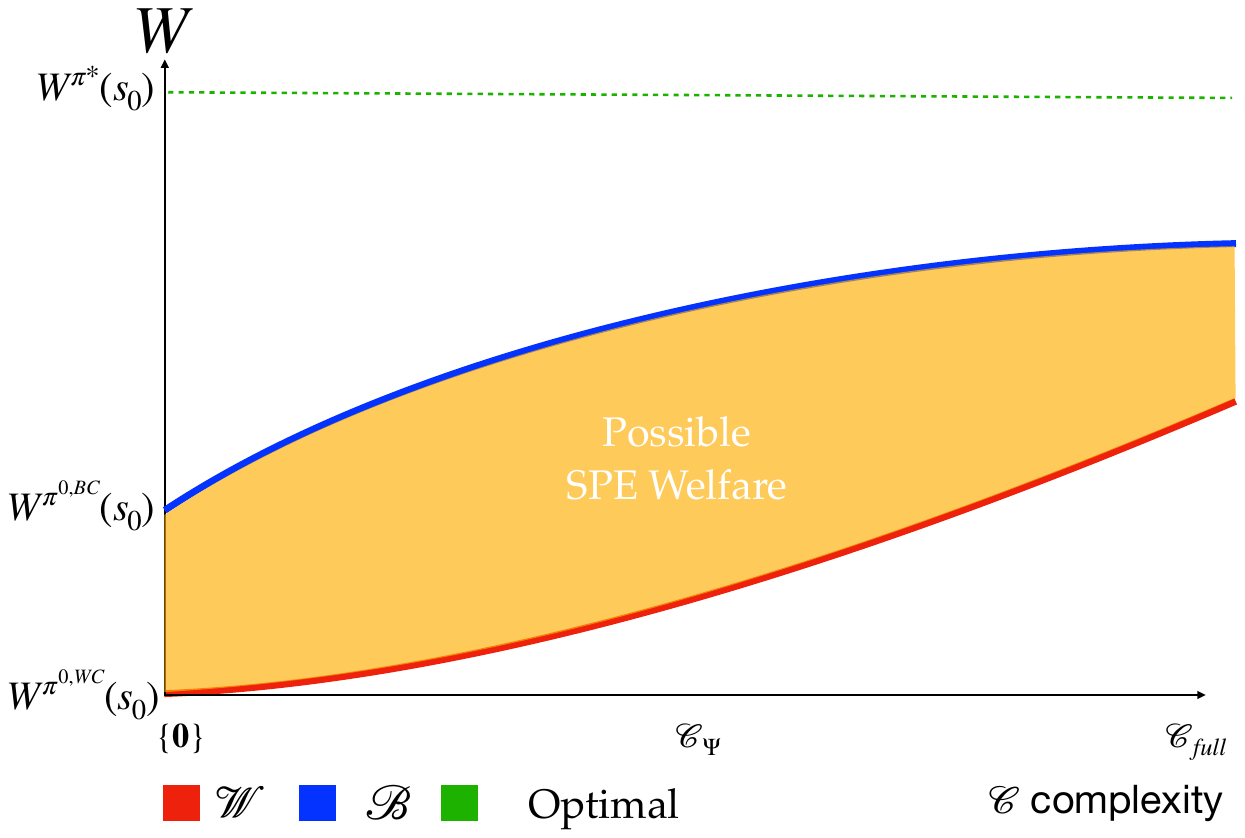}
    \caption{Under general partial observability, possible welfare levels at equilibrium improve, without necessarily reaching optimal social welfare, assuming arbitrary unconditional transfers.} 
\label{fig:partial_obs}
\end{figure}

\section{More Possible Contracts Improve Welfare}\label{sec:features}

Beyond full observability or detectable deviations, contracts are still applicable. We show that both in FOMGs and POMGs, system performance as measured by welfare improves for more expressive contracting models in the sense of $\preceq_{\text{ETE}}$.

Restrictions on the capability to contract (same contractible observation model, but restriction on the contracts) correspond to some contracts being impermissible in the environment. Reasons for such a restriction in practice may be desirable to incorporate domain knowledge into the contract design, or simply to improve exploration landscape for reinforcement learning.

In defining our order on system performance, we use $M+\btheta$ as the game following the acceptance of a contract $\btheta$. Note that the system performance definition is complicated by the existence of potentially multiple equilibria, compare \cite{folktheorem}. We can show, though, that the performance of the worst \emph{and} the best stable outcome, that is, a variant of welfare of the worst and best subgame-perfect equilibria, improve. As it might be possible that the proposing agent has an optimal contract they would like to propose, e.g. if there is a limit point of utility for them that is not achievable, we define a notion of best and worst equilibrium that considers cases in which the proposer takes an arbitrarily close to optimal contract. For finite contract spaces, this issue disappears, and the reader can view the following results as being about the worst and best subgame-perfect equilibria of $M^{\mathcal C}$.
\begin{definition}[WCSPW and BCSPW]
    Let $M$ be a Markov game, and $\mathcal{C} = \langle O_0, \bld\Theta\rangle$ be a contracting model. For $\btheta \in \bld\Theta$, we define the \emph{worst-case} resp. \emph{best-case subgame-perfect welfare under $\btheta$} as
    \begin{equation}
    \begin{split}
    \operatorname{WCSPW}_\mathcal{C}(\btheta) &\coloneqq  \inf_{\bld\pi \in \mathrm{SPE}(M+\btheta)} W^{M+\theta, \bld\pi}(s_0)\\
    \operatorname{BCSPW}_\mathcal{C}(\btheta) &\coloneqq  \sup_{\bld\pi \in \mathrm{SPE}(M+\btheta)} W^{M+\btheta, \bld\pi}(s_0).
    \end{split}\label{eq:innersupattained}
    \end{equation}
    We also define the supremum over the contract choices.
    \begin{equation}
    \begin{split}
   \mathcal{W}(\mathcal{C}) &= \sup_{\btheta \in \Theta} \operatorname{WCSPW}_\mathcal{C}(\btheta)\\
    \mathcal{B}(\mathcal{C}) &= \sup_{\btheta \in \Theta} \operatorname{BCSPW}_\mathcal{C}(\btheta).\nonumber
    \end{split}\label{eq:attainmentofsupremal} 
    \end{equation}
\end{definition}
These are bounds on the welfare under subgame perfection, which is attained if the proposer's choice of a contract has an optimal choice. Note that otherwise a subgame-perfect equilibrium fails to exist, as the following proposition shows.
\begin{proposition}[$\mathcal{W}, \mathcal{B}$ bound equilibrium welfare]\label{thm:autspace}
    Let $\bTheta$ admit arbitrary unconditional transfers, augmenting a Markov game $M$. Then, we have that, for any $\bpi \in \mathrm{SPE}(M^{\mathcal{C}})$
    \[
    \mathcal{W}(\mathcal{C}) \le W^\pi(s_0) \le \mathcal{B}(\mathcal{C}).
    \]
    If $\sup_{\btheta \in \bTheta} \operatorname{WCSPW}_\mathcal{C}(\btheta)$ resp. $\sup_{\btheta \in \bTheta} \operatorname{BCSPW}_\mathcal{C}(\btheta)$ are attained, and moreover for such $\btheta^*$ we can take $\bpi^*$ attaining $\operatorname{WCSPW}_\mathcal{C}(\btheta^*)$ resp. $\operatorname{BCSPW}_\mathcal{C}(\btheta^*)$, then there is $\bpi \in \SPE(M^{\mathcal C})$ such that $W^{\bpi} (s_0) = \mathcal W (\mathcal C)$ resp. $W^{\bpi} (s_0) = \mathcal B (\mathcal C)$. 
\end{proposition}
\begin{proof}[Proof Sketch.]
    We sketch this proof for first inequality. The second is analogous. 
    
   We first show the second assertion assuming that the supremum in \eqref{eq:innersupattained} and \eqref{eq:attainmentofsupremal} each are attained. Let $\btheta^* \in \argmax_{\btheta \in \bTheta} \mathrm{WCSPW}_\mathcal{C}(\btheta)$ and denote by $\bpi^*$ be an SPE attaining \eqref{eq:innersupattained} and $\bpi$ any subgame-perfect equilibrium in $M$. We construct an SPE for $M^{\mathcal{C}}$, by first defining $\btheta^*_{\text{tr}}$ (tr referring to \enquote{transfer}) as
    \[
    \theta^*_{\text{tr}}(o) = \begin{cases}
       \mathbf{v} & o = \mathbf{\operatorname{acc}}\\
       \theta^*(o) & o \neq\mathbf{\operatorname{acc}}
    \end{cases}
    \]
    where the vector $\mathbf v$
    \[
    \mathbf v = \begin{pmatrix}
        \sum_{i=2}^n V_i^{M+\btheta^*, \bpi^*}(s_0) - V_i^{M, \bpi}(s_0) \\
        V_2^{M, \bpi^*}(s_0) - V_2^{M+\btheta^*, \bpi}(s_0) \\
        \vdots \\
        V_n^{M, \bpi^*}(s_0) - V_n^{M+\btheta^*, \bpi}(s_0)
    \end{pmatrix}
    \]
    We have $\btheta^*_{\text{tr}} \in \bTheta$ as $\bTheta$ admits arbitrary unconditional transfers. It is an SPE as no agent at no point has an incentive to deviate. All agents except for the proposing agent are compensated exactly to the point of indifference; the proposing agent maximizes its reward, under the constraint of agents accepting the contract. Now, taking an arbitrary SPE, we have that \textit{full surplus extraction}, i.e. unconditionally transferring all excess reward from the non-proposers to the proposer, must occur at every SPE (otherwise, the proposing agent could unilaterally change by simply performing this extraction, and changing nothing else). But this directly ties the proposer's reward at all SPE to the welfare attained across all agents: hence, in order to get better payoff than $\mathcal{W}(\mathcal{C})$ at equilibrium, the welfare attained at equilibrium must be higher---hence, $\mathcal{W}(\mathcal{C})$ acts as a lower bound on the possible welfare attained across all SPE of the contracting game. 

    By choosing $\btheta^*$ to be close to $\sup_{\btheta \in \bTheta} \mathrm{WCSPW}_\mathcal{C}(\btheta)$ we can generalize to the case where the exact value is not attained.
\end{proof}
The main theorem of this section is the following monotonicity result.
\begin{theorem}[Monotonicity in Contract Spaces]\label{thm:monotonicity}
    We have that the bounds $\mathcal{W}(\mathcal{C})$ and $\mathcal{B}(\mathcal{C})$ are both monotonically increasing quantities in $\mathcal{C}$: that is, when $\mathcal{C}_1 \preceq_{\operatorname{ETE}} \mathcal{C}_2$, we have $\mathcal{W}(\mathcal{C}_1) \le \mathcal{W}(\mathcal{C}_2)$ and $\mathcal{B}(\mathcal{C}_1) \le \mathcal{B}(\mathcal{C}_2)$
\end{theorem}

\begin{proof}
    Let $\mathcal{C}_1 \preceq_{\operatorname{ETE}} \mathcal{C}_2$. We show the proof only for $\mathcal{W}$: the proof for $\mathcal{B}$ is identical. In order to establish the desired inequality, given that the expression
    
    \[
    \mathcal{W}(\mathcal{C}) = \sup_{\theta \in \Theta} \mathrm{WCSPW}_\mathcal{C}(\theta)
    \]

    are defined by suprema, it is enough to show that, for any $\theta_1 \in \Theta_1$, there exists a $\theta_2 \in \Theta_2$ such that $\mathrm{WCSPW}_{\mathcal{C}_1}(\theta_1) \le \mathrm{WCSPW}_{\mathcal{C}_2}(\theta_2)$. But by definition of $\preceq_{\operatorname{ETE}}$, for any $\theta_1 \in \Theta_1$, there exists a $\theta_2 \in \Theta_2$ which matches expected transfers at all state-action transitions---hence all payoffs in the base POMG are matched exactly by such $\theta_2 \in \Theta_1$ at all states and observation histories. Hence, all subgame-perfect equilibria attainable under $\mathcal{C}_1$ are attainable in $\mathcal{C}_2$, as needed. 
\end{proof}

The implications of this result are illustrated in \Cref{fig:full_obs} and \Cref{fig:partial_obs}---specifically, whenever contracting models are nested in expressiveness, welfare in subgame-perfect equilibrium increases, whether or not optimal performance is attained in the limit of complexity. We apply \Cref{thm:monotonicity} in special cases.

\subsection{Deterministic Contractible Features}

Here, we formalize the relationship between contracting models and contracting on a set of features $\phi$ of the state $s \in S$. Let $\psi: S \to \Psi$ be a generic (deterministic) feature map. We define the feature-based contracting model on feature mapping $\bld \psi$ through the observation model $O_0 (s, \mathbf a) = \psi (s)$, when paired with the space of all $\psi$-dependent contracts $\Theta_{\operatorname{full}}$, and denote it by $\mathcal C_{\psi} = (\bTheta_{\psi}, \psi)$.
\begin{corollary}[Feature Monotonicity]\label{thm:feature_monotonicity}
    Let $M$ be a Markov game and $\psi_1, \psi_2$ two feature models such that $\kappa: \Psi_2 \to \Psi_1$ and $\kappa \circ \psi_2 = \psi_1$. Then,
    \begin{align*}
    \mathcal{W}(\mathcal{C}_{\psi_1}) &\le \mathcal{W}(\mathcal{C}_{\psi_2})\\
    \mathcal{B}(\mathcal{C}_{\psi_1}) &\le \mathcal{B}(\mathcal{C}_{\psi_2}).
    \end{align*}
\end{corollary}
\begin{proof}
    By \Cref{thm:monotonicity}, it suffices to show that $\mathcal{C}_{\psi_1} \preceq_{\operatorname{ETE}} \mathcal{C}_{\psi_2}$. Let $\btheta_1 \in \bTheta_{\psi_1}$. We have that 
    \begin{align*}
        \btheta_{1} \circ \psi_1 (s, \mathbf a) 
        &= \btheta_{1} \circ (\kappa \circ \psi_2) (s, \mathbf a)\\
        &= (\btheta_{1} \circ \kappa) \circ \psi_2 (s, \mathbf a).
    \end{align*}
    Hence, in the definition of $\preceq_{\text{ETE}}$, we can choose $\btheta_{2} = \btheta_{1} \circ \kappa$ as a contract in $\Theta_{\psi_2}$. As it is pointwise the same contract, also the implied payoffs in \Cref{eq:expressiveness}.
\end{proof}
This result gives that, by \Cref{thm:autspace}, adding features to your contracting space shifts the region of possible equilibrium welfares upwards. This motivates the monotonicity in \Cref{fig:full_obs} and \Cref{fig:partial_obs}. 

\subsection{Random Contractible Features in a Fully Observable Markov Game}
Noisy features still give rise to monotonicity. Concretely, let $\psi: S \times Z \to \Psi$ be a mapping for a random variable $Z$. This gives rise to a contractible observation model $O_0 (s, \mathbf a) = \psi (s, z)$, which may be random.
\begin{proposition}[Monotonicity of Random Features]\label{thm:randomfeat}
    Let $M$ be an FOMG, and suppose we have that $s \to \psi_2 \to \psi_1$ form a Markov chain. Then, $\mathcal{W}(\mathcal{C}_{\psi_1}) \le \mathcal{W}(\mathcal{C}_{\psi_2})$ and $\mathcal{B}(\mathcal{C}_{\psi_1}) \le \mathcal{B}(\mathcal{C}_{\psi_2})$. 
\end{proposition}
Note that this result subsumes the deterministic case, as the second arrow in the above Markov chain takes the role of $\kappa$ in that definition.
\begin{proof}
    We proceed, analogously to the deterministic feature case, by showing that any random contract $\btheta_1 \in \bTheta_{{\psi}_1}$ has expected transfers matched by some contract $\btheta_2 \in \bTheta_{{\psi}_2}$, hence verifying $\mathcal{C}_{\psi_1} \preceq_{\operatorname{ETE}} \mathcal{C}_{\psi_2}$, and then applying \Cref{thm:monotonicity} for the desired result. 

    Let $\btheta_1 \in \bTheta_{{\psi}_1}$. Then, since $S \to {\psi_2} \to {\psi_1}$ is a Markov chain, we know there exists a Markov kernel $\kappa$ s.t. $\psi_1 = \kappa \circ \psi_2$. From this, we define the contract $\btheta_{2} \in \bTheta_{\psi_2}$ by the following expectation
    \[
    \btheta_{2} \coloneqq  \mathbb{E}[\btheta_1|\psi_2],
    \]
    which is a function only of $\phi_2$. This contract choice is as desired since, by the tower property of conditional expectation:
    \begin{align*}
        \mathbb{E}_{\phi_2 \sim \psi_2(s, \mathbf{a})}[\btheta_2(\phi_2)] = \mathbb{E}_{\phi_1 \sim \psi_1(s, \mathbf{a})}[\mathbb{E}_{\phi_2 \sim \kappa(\phi_1)}[\btheta(\phi_2)|\phi_1]] = \mathbb{E}_{\phi_1 \sim \psi_1(s, \mathbf{a})}[\btheta_1(\phi_1)]
    \end{align*}
    Hence, $\mathcal{C}_{\bld \psi_2} \preceq_{\operatorname{ETE}} \mathcal{C}_{\bld \psi_1}$, and we can apply \Cref{thm:monotonicity}, as needed.
\end{proof}

\subsection{Random Contractible Features in a Partially Observable Markov Game}

Note that, in no step for the above sections do we rely on full observability for indiviual agents. Moreover, the case of random features in a deterministic game is subsumed by different choices of contractible observations $O_0$, and so no further extension is required on that front. Therefore, we have the fully general result:

\begin{theorem}\label{thm:generalpomg}
     Let $P = \langle S, s_0, O,  \mathbf  A, T, \mathbf{O}, \mathbf R, \gamma\rangle$ be a POMG, and let $\mathcal{C} = \langle O_0, \Theta \rangle$ be a contracting model. Then, the conclusions in \Cref{thm:autspace}, \Cref{thm:feature_monotonicity}, \Cref{thm:randomfeat} all continue to hold for $P^\mathcal{C}$, when $\operatorname{WCSPW}_\mathcal{C}$ and $\operatorname{BCSPW}_\mathcal{C}$ minimize and maximize over the space of all equilibria of \textit{recurrent} policy profiles $\Pi: H_i \to \Delta A_i$, and with contracts depending on contracting observation $O_0$. 
\end{theorem}

\subsection{Contracting in History-Transparent POMGs and Decentralized Partially Observable Markov Decision Processes}

As depicted \Cref{fig:partial_obs}, it is not true that arbitrarily expressive contracting in general POMGs will give the optimal policy profile performance. If we ignore strategic conflicts between agents, one reason for this is coordination -- agents may have well-shaped incentives to act on if provided the full state, but may not be able to effectively act on them in a coordinated way, simply due to a lack of common observation. Since contracting can only mitigate incentive issues, we cannot hope to ameliorate this with contracting.

A model more suitable to this problem is a well-known restriction to POMGs---the Decentralized Partially Observable Markov Decision Problem (DEC-POMDP) \citet{bernstein2002complexity}. In DEC-POMDPs, all agents are given an identical scalar reward to maximize, but different observations. Does (sufficiently expressive) contracting make RL on POMGs behave like RL in DEC-POMDPs? 

To establish this connection, we need to make one more simplifying assumption on the considered POMGs -- that all agents' observation-action histories are an almost-surely deterministic function of the current true state. We call POMGs satisfying this assumption \textit{history-transparent} POMGs. Assuming this, it is possible to implement a forcing contract for the optimal DEC-POMDP policy under $\mathcal{C}_{\operatorname{full}}$ -- hence, by \Cref{thm:generalpomg}, all SPE of such a POMG augmented with $\bld \Theta_{\operatorname{full}}$ attain optimal reward under a DEC-POMDP with equivalent transition dynamics, where the shared reward is simply the social welfare. One can view this as a reduction of $\mathcal{C}_{\operatorname{full}}$-contracting in history-transparent POMGs to solving a DEC-POMDP at equilibrium. We leave this analysis for future work.

\subsection{Necessity of Arbitrary Unconditional Transfers}

As a last result in this section, we show that our restriction to contract spaces that admit arbitrary unconditional transfers was necessary. Welfare does not necessarily increase when moving to $\preceq_{\text{ETE}}$-greater contract spaces when the spaces involved do not admit arbitrary unconditional transfers.

\begin{example}[Bribing the Contract Proposer]
    Recall the Prisoner's Dilemma of \Cref{subfig:pd}, with the row player as the contract proposer. We can consider $\bTheta = \{\mathbf{0}, \btheta\}$ (recalling the null contract $\mathbf{0}$ having no effect on the original game), where $\btheta$ imposes the choice (C, C) for both agents with high penalty (say, -10) of defecting, with an additional unconditional transfer of 0.5 from the proposer to the column player, and $\bTheta^\prime = \{\mathbf{0}, \btheta, \btheta'\}$, where $\btheta'$ imposes a penalty of $-10$ on only the column player for defecting, and a transfer of $1.25$ for cooperating. Then, for $\bTheta$, we get the proposer has an incentive to select $\btheta$ over null contract $\mathbf{0}$, attaining reward $-1.5$ at equilibrium, but in the second space, the proposer has a stronger incentive to select $\btheta'$, where it gets a reward of $-1.25$ at equilibrium (in all cases, these contracts give reward greater than $-2$ for both proposing and accepting agents, and hence are incentive-compatible in proposal, acceptance, and in selection at the given equilibrium). Thus, even though $\bTheta'$ has more possible contracts to choose from, worse welfare are attained in the limit due to poor incentives for the proposing agent.
\end{example}

\section{Evaluation}\label{sec:methodology}
We now evaluate the performance of the contracting augmentation. First, we introduce the baseline methods that we use to evaluate our approach. Then, we introduce our experimental domains. Finally, we provide details on MOCA, our training procedure for contracting. Replication code can be found at \url{https://github.com/Algorithmic-Alignment-Lab/contracts}.

\subsection{Evaluation}
We evaluate MOCA by comparing to the following baselines.
\begin{tightlist}
\item Joint Training: a centralized algorithm with joint action space $\mathbf A = \times_{i=1}^N A_i$ chooses actions to maximize welfare;
\item Separate Training: Agents selfishly maximize their reward;
\item Gifting: Agents can \enquote{gift} \cite{lupu2020gifting} another agent at every timestep by directly transferring some of their reward;
\item Vanilla Contracting: Run an off-the-shelf deep RL on the contract-augmented versions of the respective domains.
\end{tightlist}
We train all domains with 2, 4, and 8 agents, using Proximal Policy Optimization (PPO) \cite{schulman2017proximal} with continuous state and action spaces with Gaussian sampling in ray rllib's \cite{liang2018rllib} implementation (hyperparameter choices can be found in \Cref{apx:hyperparameters}). In each domain, we train gifting agents with a lower bound of $0$ and an upper bound on transfer value in contracts. This allows the same magnitude of transfers in gifting and contracts, for fair comparison. In one of the games, Emergency Merge, we reduced the gifting values to $10$ per timestep, as this improved gifting's performance. On the Prisoner's Dilemma and the Public Goods game, we trained agents for 1M environment steps, and the complex dynamic games are trained for 10M environment steps.

\subsection{Games}
We test on several classes of games. We use Prisoner's Dilemma and a Public Goods game as static, simultaneous-move games, and Harvest, Cleanup, and Emergency Merge as dynamic domains.

\begin{description}
    \item[Prisoner's Dilemma.] The $n$-agent Prisoner's Dilemma has two actions per agent, cooperate and defect. If all agents cooperate, they each get reward $n$, and if all defect, they all get reward $1$. However, if some defect and some cooperate, the ones that cooperate get reward $0$ and the ones that defect get reward $n+1$. Again, the socially optimal outcome is the one where all agents cooperate, but only Nash equilibrium is where all agents defect. We run an additional timestep after the matrix game is played for gifting actions to take place.
    \item[Public Goods.] We study the following public goods game \cite{janssen2003adaptation}. Agents choose an \emph{investment} $a_i \in [0, 1]$, and get reward $R_i(\mathbf a) =  \frac{1.2}{N} \sum_{j=1}^N a_j - a_i$, i.e. they are given their share of the public returns, the investment returning 20\%, minus their own investment level. At social optimum, all agents choose $a_i=1$ to get optimal joint reward. However, selfish agents are not incentivized to invest at this high level, as they would like to free-ride on the other agents' efforts. 
    \item[Harvest.] In Harvest, from \citet{hughes2018inequity}, agents move along a square grid to consume apples, gaining a unit of reward. Apples grow faster if more apples are close by, which leads to incentives to overconsume now, leading to an intertemporal dilemma. We choose engineered features to limit the amount of computational resources needed. In particular, agents receive their position and orientation, the coordinates and orientation of the closest other agent, the position of the nearest apple, the number of apples close to the agent, the total number of apples, and the number of apples eaten by each agent in the last timestep. We don't allow agents to use a punishment beam following \citet{lupu2020gifting}. The environment runs for 1,000 timesteps per episode. 
    \item[Cleanup.] In Cleanup, also from \cite{hughes2018inequity}, agents similarly move along a square grid to consume apples and gain one unit of reward. Apples only spawn if a nearby river contains a number of waste objects lower than a threshold. Removing a waste object is a costless, but also rewardless, task. Apple-picking agents can free-ride on other agents, which leads to degraded performance. The observation space used for agents is simplified to limit computational requirements, and agents are passed their position and orientation, the position and orientation of the closest agent, the positions of the closest apple and waste object, and the number of current apples and waste objects. The environment runs for 1,000 timesteps per episode. 
    \item[Emergency Car.] A set of $n-1$ cars approaches a merge, an ambulance behind them, compare \Cref{fig:driving}. The ambulance incurs a penalty of $100$ per timestep that it has not reached the end of a road segment past the merge. The cars in front also want to get to the end of the road segment, but incur a penalty of only $1$ per timestep. They are limited to one-fourth of the velocity that the ambulance can go. We assume access to controllers preventing cars from colliding (stopping cars short of crashing into another car) and managing merging, and so the actions $a_i \in [-0.1, 0.1]$ only control the forward acceleration of each vehicle. A dilemma arises as cars prefer to drive to the merge fast, not internalizing the strong negative effect this has on the ambulance. The environment resets after 200 rounds or when cars crash, whichever is earlier. Note that here, due to the asymmetry of agent capabilities and rewards, attaining optimal social welfare cannot be done via Pareto improvement within the original game.
\end{description}

\begin{figure}
    \centering
    \includegraphics[width=.6\linewidth]{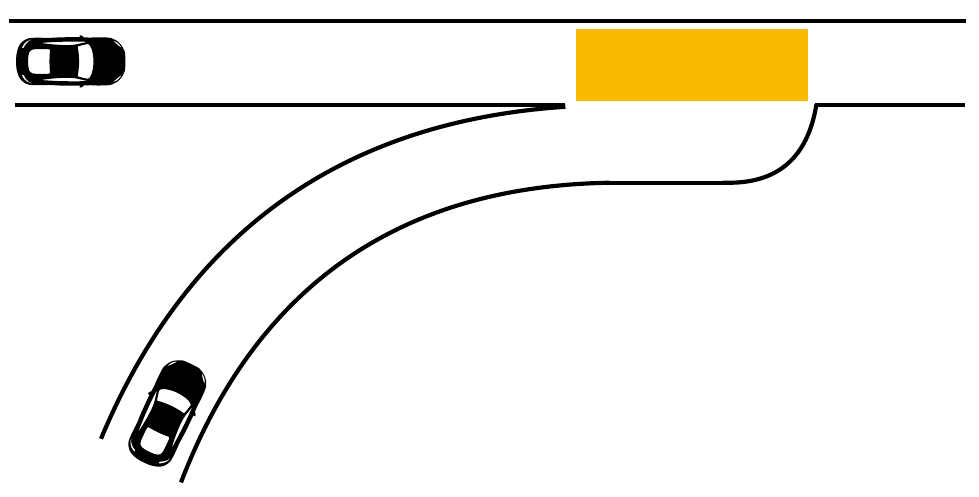}
    \caption{A depiction of the emergency merge domain.}
    \label{fig:driving}
\end{figure}

\subsection{Contract Spaces}
We consider low-dimensional contract spaces for different domains. 
\begin{description}
\item[Prisoner's Dilemma] Contracts are parameterized by a transfer $\theta \in [0, n]$ for defecting, which is distributed to the other agents in equal proportions.
\item[Public Goods] Contracts are parameterized by a transfer $\theta \in [0,1.2]$, agents transfer $\theta (1-a_i)$, which is distributed to the other agents in equal proportions.
\item[Harvest] Contracts are parameterized by $\theta \in [0, 10]$. When an agent takes a consumption action of an apple in a low-density region, defined as an apple having less than 4 neighboring apples within a radius of $5$, they transfer $\theta$ to the other agents, which is equally distributed to the other agents.
\item[Cleanup] Contracts are parameterized by $\theta \in [0, 0.2]$, which correspond to a payment per garbage piece cleaned, paid for evenly by the other agents.
\item[Emergency Car] The ambulance can propose a per-unit subsidy of $\theta \in [0, 100]$ to the cars at the time of ambulance crossing. Each car is transferred $\theta$ times its distance behind the ambulance at time of merge by the ambulance. If a car is ahead of the ambulance at time of reward, it pays the ambulance $\theta$ times its distance ahead of the ambulance. 
\end{description}

\subsection{Training}

The contracting augmentation yields a Markov game, for which one could directly train agents with deep reinforcement learning (we will call this \emph{vanilla contracting}). However, as can be observed from \Cref{fig:linesmatrix} and \Cref{fig:matrix}, this implementation of contracting does not outperform joint training in problems with more complex dynamics, or higher-dimensional state and action spaces. To fix this, we propose an algorithm inspired by multi-objective reinforcement learning, compare \cite{andrychowicz_hindsight_2017}, \emph{Multi-Objective Contract Augmentation Learning} (MOCA). We present it in \Cref{alg:pseudocode}. MOCA consists of two phases: first, the algorithm draws random contracts (which, in the language of multi-objective reinforcement learning, can be viewed as different \enquote{objectives}). This can be used to estimate $V^{\boldsymbol \pi}_i (s_0, \boldsymbol\theta)$, $i=1, 2, \dots, n$ for the initial state $s_0$ and any contract $\boldsymbol\theta$, i.e. the values for agents when contract $\boldsymbol\theta \in \boldsymbol\Theta$ is in force. This allows it to learn  estimates of the utility agents will get under a particular contract. Due to random sampling, these estimates are not biased by contract exploration, which may be an issue when using deep reinforcement learning directly.

In a second phase, we freeze play following $(s_0, \boldsymbol\theta)$ for any contract $\boldsymbol\theta$ and the policy at states $(i, \boldsymbol 0)$ and $(i, \boldsymbol\theta)$. We do so by choosing a contract repeatedly from the policy $\pi_i(i, \boldsymbol 0)$, and use as a proxy for acceptance the expected probability of acceptance, $\prod_{j=1}^n \pi_j(i, \boldsymbol\theta)$. In order to help exploration of the contract space in this stage, we sample $\boldsymbol \nu$ agents from the space of non-proposing agents, and only use these agent's accept-reject probabilities in determining contract acceptance. Here, the introduced $\boldsymbol \nu$ becomes a tunable hyperparameter, for which $\boldsymbol \nu = 2$ obtained strong performance across all domains, which we report in \Cref{sec:experiments}. We update the weights for the actions of all agents at $\pi_i(i, \boldsymbol 0)$ and $\pi_j (i, \boldsymbol\theta)$, for $j=1, 2, \dots, n$. Finally, the algorithm returns the so-obtained policy profile. 

\begin{algorithm}
\caption{Multi-Objective Contract-Augmentation Learning (MOCA)}\label{alg:pseudocode}
\KwData{Contract Space $\boldsymbol\Theta$ including the null contract $\boldsymbol 0$, Markov Game $M$, probability distribution $P(\boldsymbol\Theta)$}
\KwResult{Policy Profile $\boldsymbol\pi$}
$\boldsymbol\pi\leftarrow \operatorname{initialize\_policies}()$\;
\For{$t=1$ \KwTo $\frac{9}{10}\operatorname{num\_episodes}$}{
$\boldsymbol\theta \sim P(\boldsymbol\Theta)$\;
$\operatorname{train\_subgame\_episode}(\boldsymbol\pi (s_0, \boldsymbol\theta))$
}
Freeze $\pi|_{S \times \Theta}$\;
\For{$i=1$ \KwTo $\frac{1}{10}\operatorname{num\_episodes}$}{
$\boldsymbol\theta \sim \pi_i (i, \boldsymbol 0)$\;
\lIf{$\operatorname{rand()} < \prod_{j=1}^n \pi_j(i, \boldsymbol\theta)$}{$\textbf{contract} \leftarrow \boldsymbol\theta$}\lElse{$\textbf{contract} \leftarrow \boldsymbol 0$}
$\mathbf R \leftarrow \operatorname{sample\_episode\_reward}(\boldsymbol\pi,\textbf{contract})$\;
$\operatorname{train\_with\_rewards}(\boldsymbol\pi, \mathbf R)$\;
}
\Return{$\boldsymbol\pi$}\;
\end{algorithm}
We evaluate the performance of the final trained algorithm on rollouts. The choice of length of the two periods (e.g. the $\frac{9}{10}$th for the first phase) is arbitrary.

\subsection{Monotonicity}
Finally, in addition to assessing the performance of contracting (vanilla and MOCA) on these domains, we also wish to empirically validate the monotonicity analysis for policies shown in \Cref{sec:features}. To study this, we considered Cleanup with 4 agents, and explored whether or not \Cref{thm:randomfeat} held in a complex, dynamic game. Recall that our choice of contract for this domain depends on agents being compensated for cleaning a river. To study feature monotonicity, we take the same underlying contracting logic, but introducing a probability $\alpha \in \{0,0.25,0.5, 0.75, 1\}$ that a given cleaning action is unobserved. This is equivalent to keeping $\Theta$ constant, but altering $O_0$, in the contracting model we used earlier for the Cleanup domain. Notice that $\alpha = 0$ represents is equivalent to the contracting experimented with earlier, and $\alpha=1$ is equivalent to the separate training baseline. Our theory would suggest that, because the 1-0 random features for detection of contracts form a Markov chain when $\alpha_1 > \alpha_2$, we ought to have decreasing welfare as $\alpha$ increases. 

\section{Results}\label{sec:experiments}
We first present a sample of our experiments with our baselines, which motivate the need for MOCA (\Cref{alg:pseudocode}), in \Cref{fig:linesmatrix}. Then, we discuss overall trends from all conducted experiments, \Cref{fig:matrix}. Finally, we discuss results analysing the performance of contracting under randomized features, as compared to predictions from theory, as seen in \Cref{fig:monotonic}.

\subsection{Contracting Using Deep Reinforcement Learning in the Contracting Augmentation} Consider first \Cref{fig:linesmatrix}. We observe that, in Prisoner's Dilemma and Cleanup, the baseline implementation of contracting is sufficient to achieve optimal or near-optimal performance, as can be seen by contracting either matching or surpassing the social welfare of training all agents jointly, and vastly surpassing the welfare of both gifting and separate training (both of which converge to socially suboptimal Nash equilibrium welfare). However, in more complex domains, such as Cleanup, this ceases to be the case. One potential reason for this is that, in these domains, learning the best responses to contracts becomes much more challenging, and so estimates of value for given contracts are less reliable early in training. Therefore, the proposing agent may benefit from additional exploration of the space of contracts, the main feature of MOCA. As seen in \Cref{fig:matrix}, MOCA again attains higher social welfare than joint training, separate training, and gifting. However, since intermediate levels of reward are not directly comparable with the baselines (since contracts are randomly sampled in the first stage of training, and are not run for the same number of timesteps in the second stage), MOCA is omitted from \Cref{fig:linesmatrix}. For this, results are presented in \Cref{fig:matrix} with bar plots summarizing welfare at the end of training, for all evaluated methods.

\subsection{Multi-Objective Contracting Augmentation} Now, we take a closer look at the full results in \Cref{fig:matrix}. In the simpler domains (left two columns), MOCA, like vanilla contracting, attained social welfare is vastly higher than for separately trained agents and agents trained with gifting. In Prisoner's Dilemma, contracting reaches joint optimality for 2, 4, and 8 agents. A smaller action space (and hence easier exploration) is a potential reason for why contracting can perform \emph{even better} than joint training, since the action space for joint training grows exponentially in the number of agents. In Public Goods, especially for higher number of agents, joint training interestingly outperforms MOCA, but not vanilla contracting. One possible reason for this is that, uniquely in our suite of environments, learning best responses to each contract is challenging, while the socially optimal policy is itself trivial to execute. Therefore, early in training, it is likely that socially optimal play is learned as a response to some of the contracts, particularly for those $\theta$ which are near-optimal. Therefore, biasing contract exploration early on is good for performance. In complex environments, since the socially optimal contracts are harder to execute, early biasing of contract exploration is unlikely to be well-informed, and so converging onto a poor contract proposal algorithm is likely in vanilla contracting at scale.

\begin{figure*}
\centering
\begin{tabular}{lccccc}
\toprule
\# agents & Pris. Dilemma & Public Goods & Harvest & Cleanup & Merge \\
\midrule
\includegraphics[width=.15\linewidth]{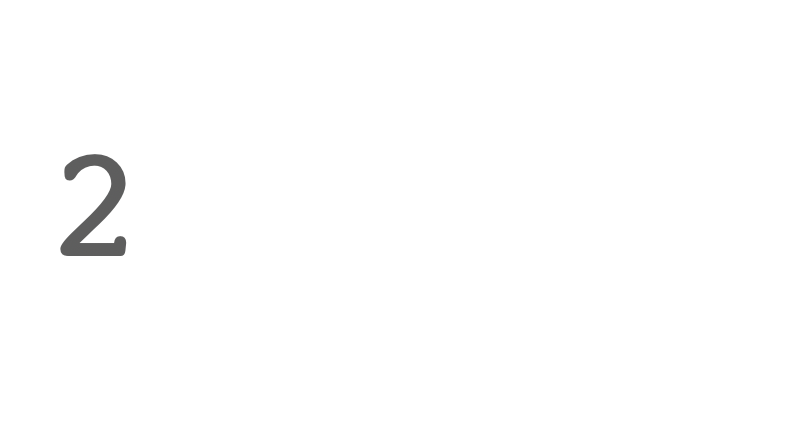}
 &\includegraphics[width=.15\linewidth]{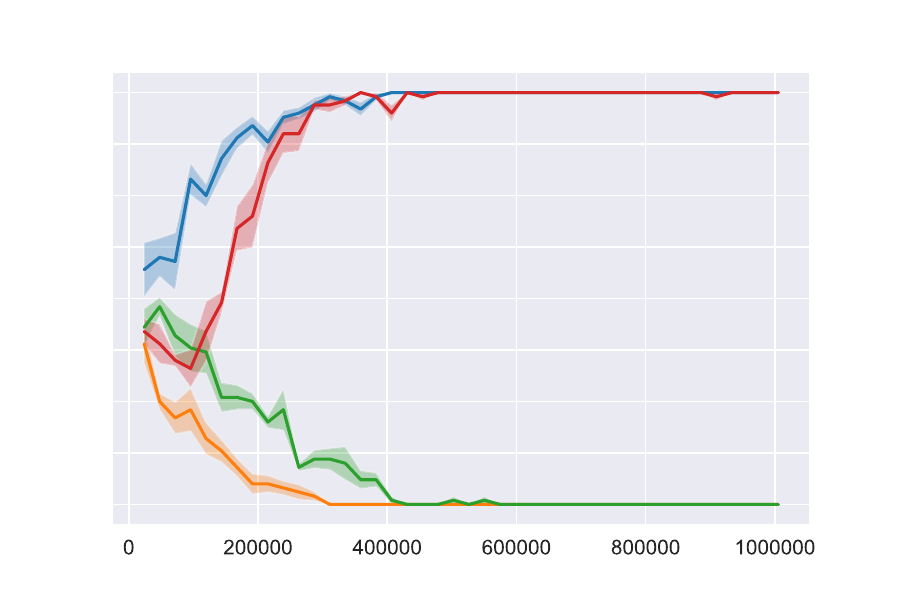} & \includegraphics[width=.15\linewidth]{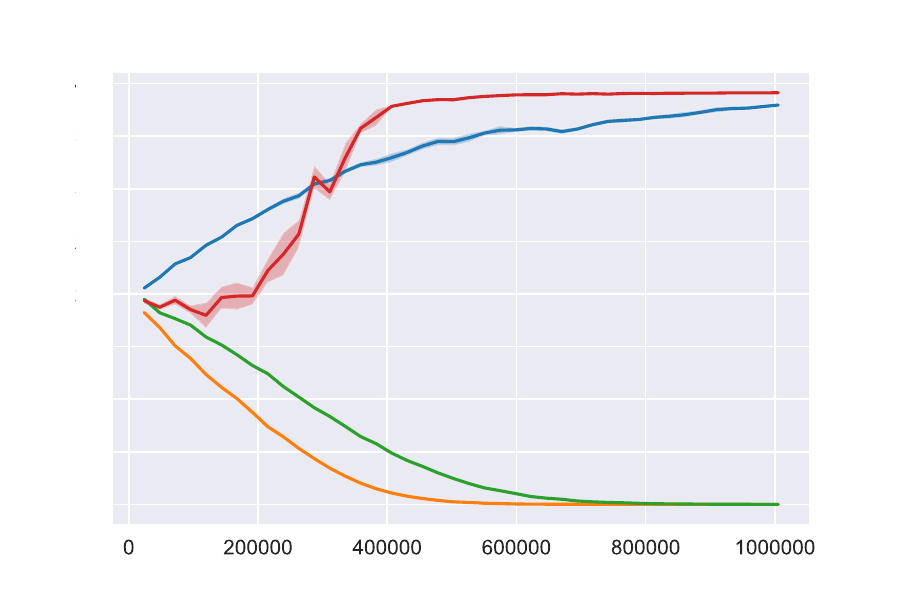}
&\includegraphics[width=.15\linewidth]{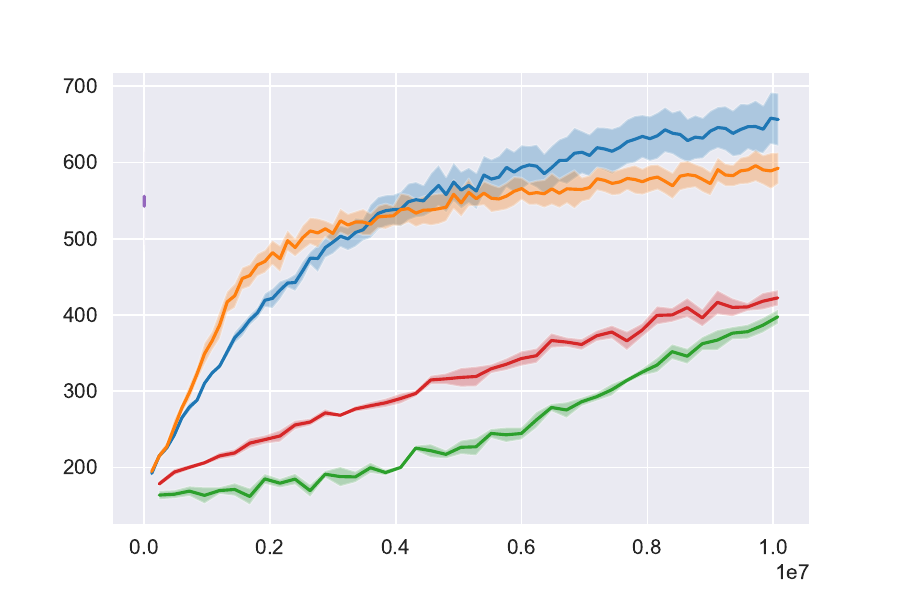}  
&\includegraphics[width=.15\linewidth]{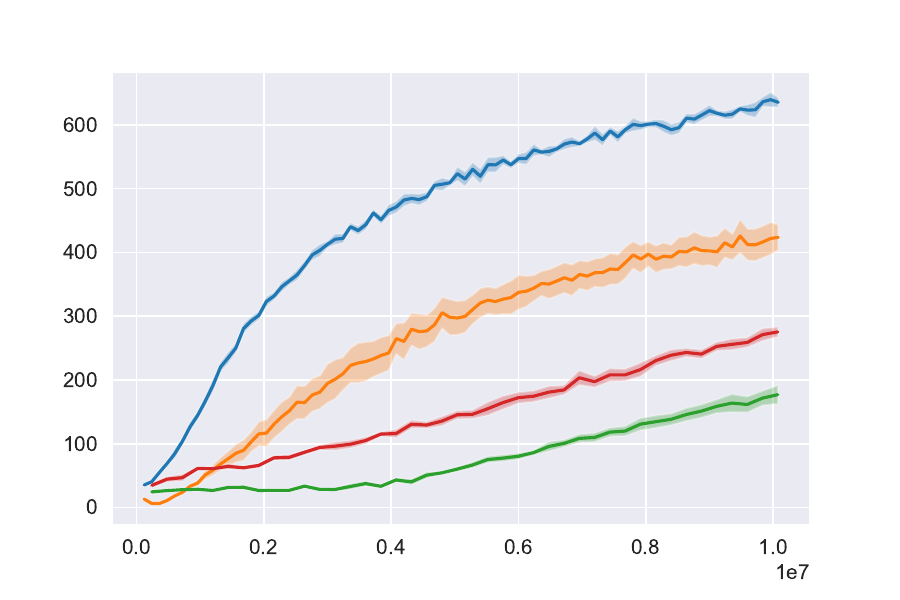}  
&\includegraphics[width=.15\linewidth]{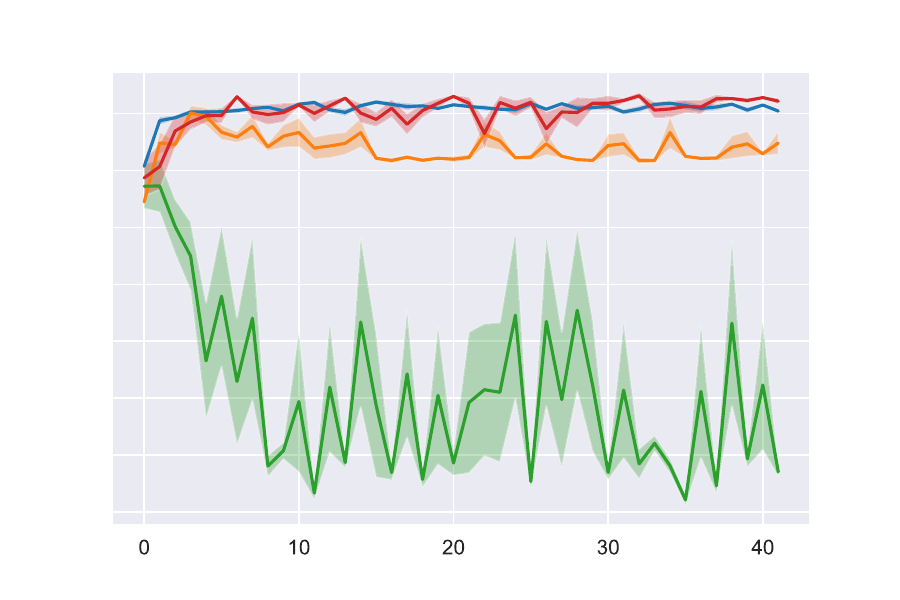}  \\
\includegraphics[width=.15\linewidth]{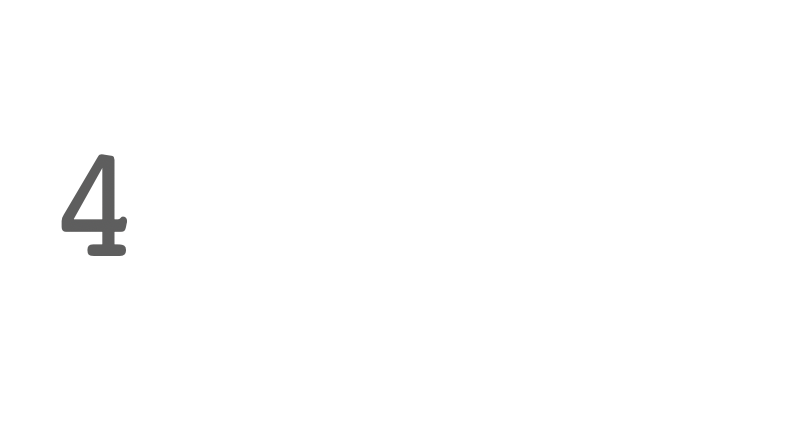}
 &\includegraphics[width=.15\linewidth]{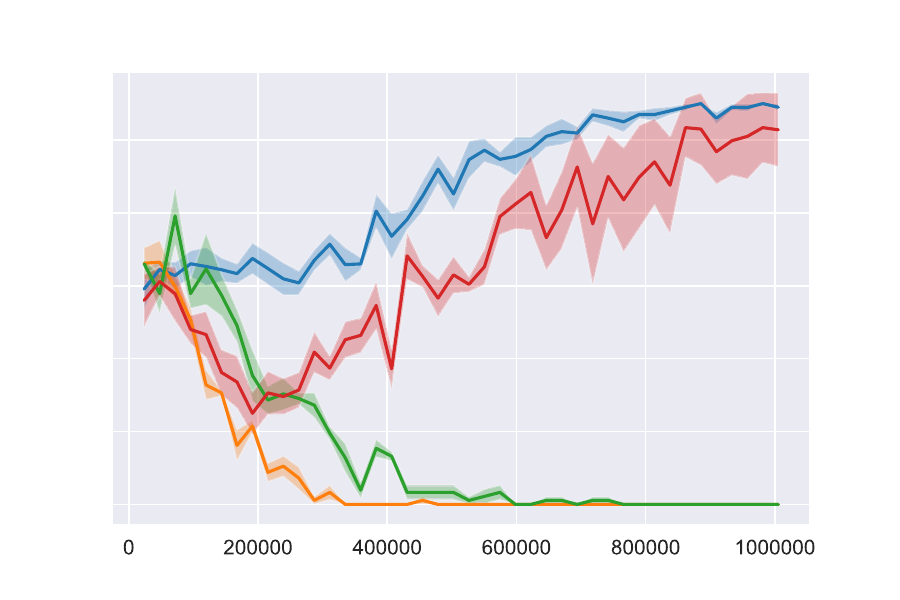} & \includegraphics[width=.15\linewidth]{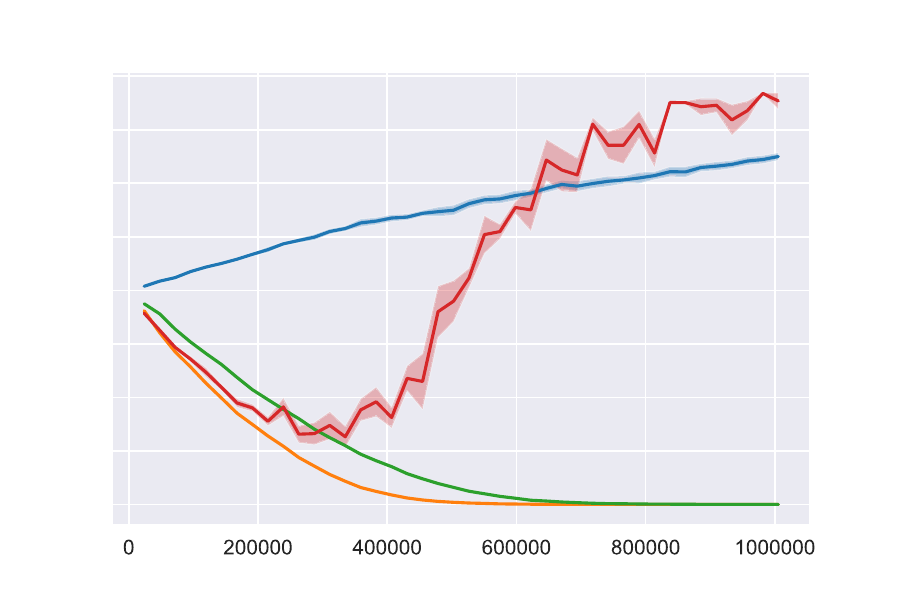}
&\includegraphics[width=.15\linewidth]{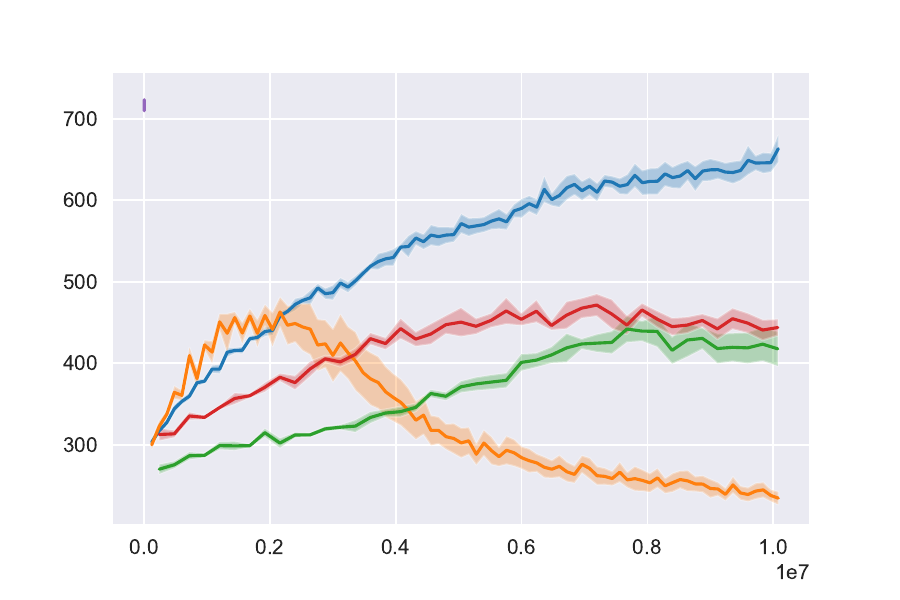}  
&\includegraphics[width=.15\linewidth]{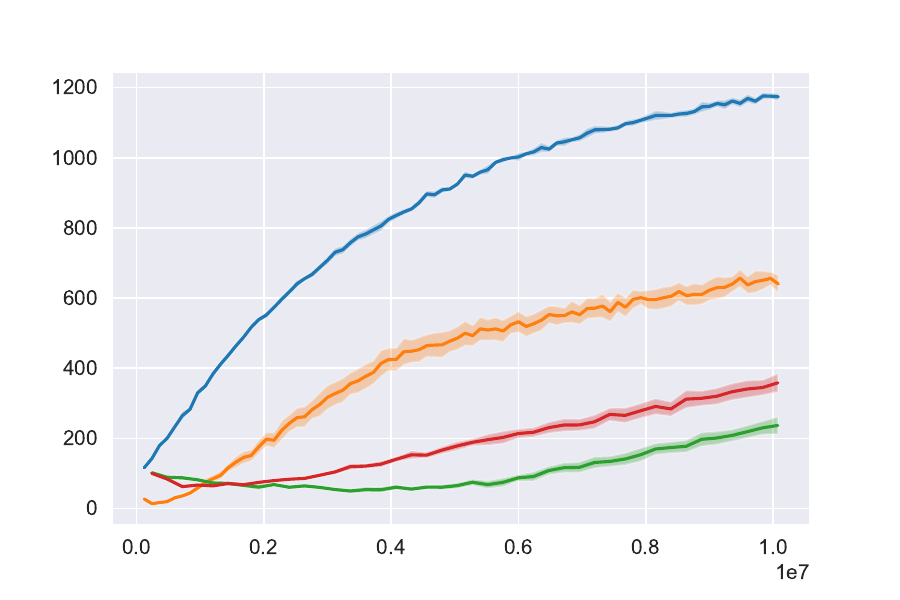}  
&\includegraphics[width=.15\linewidth]{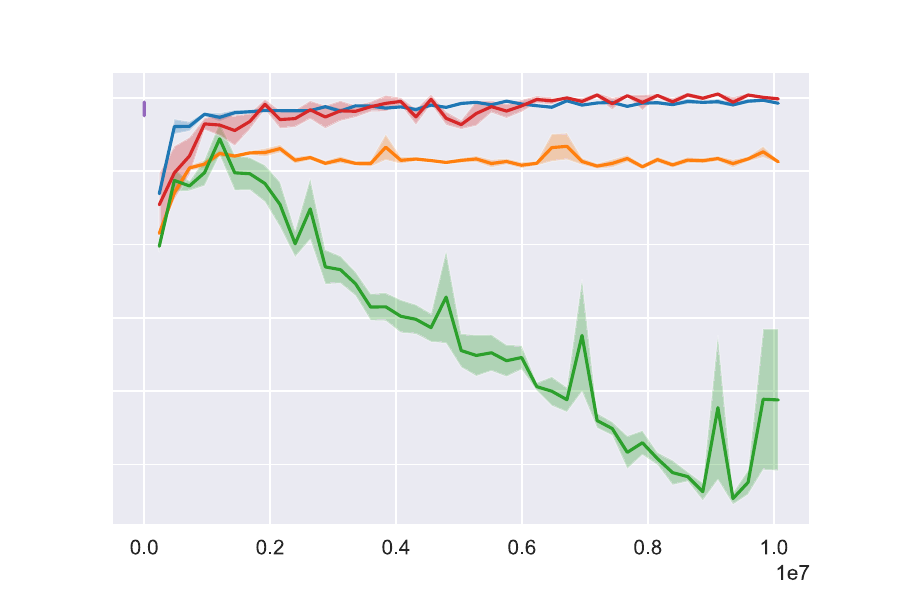}  \\
 \includegraphics[width=.15\linewidth]{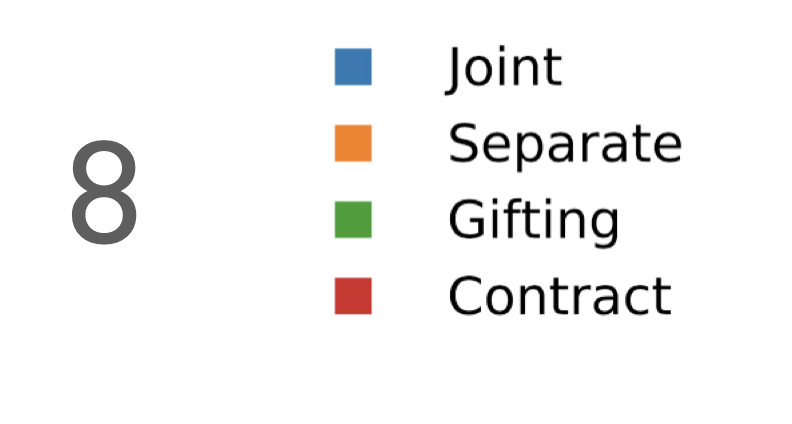}
 &\includegraphics[width=.15\linewidth]{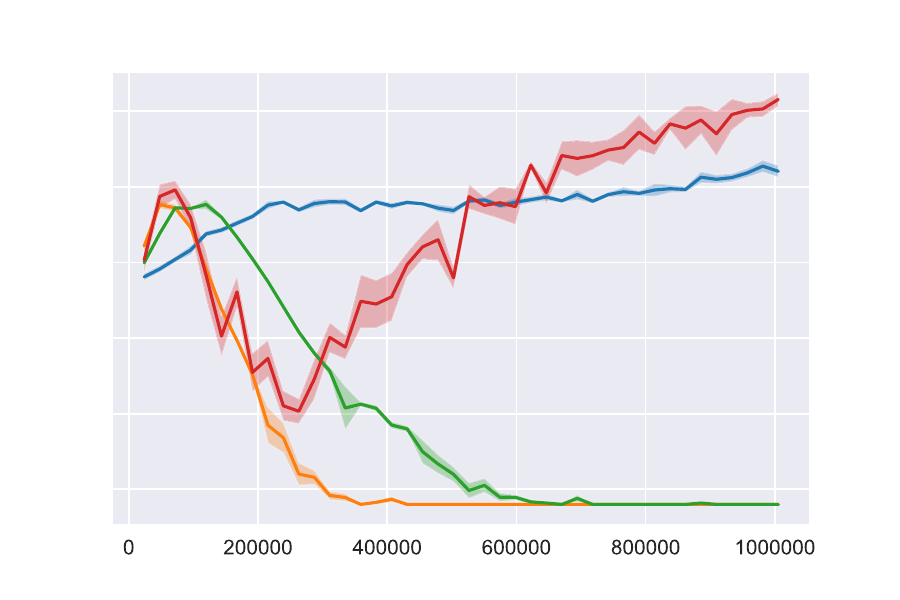} & \includegraphics[width=.15\linewidth]{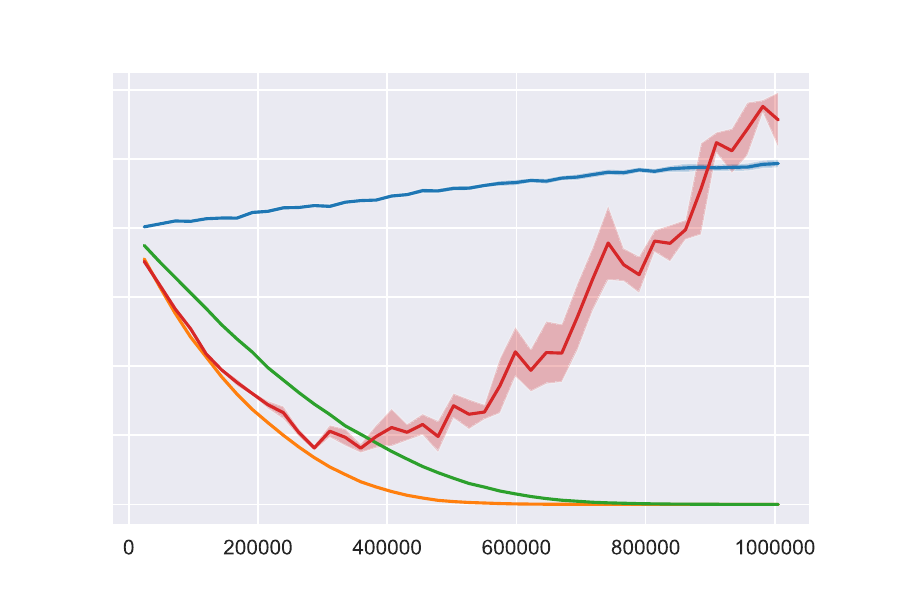}
&\includegraphics[width=.15\linewidth]{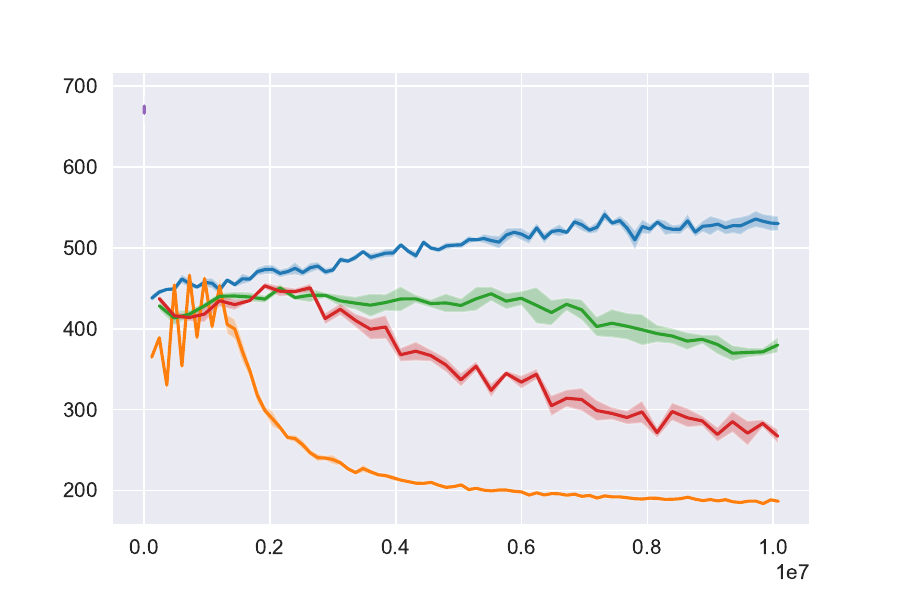}  
&\includegraphics[width=.15\linewidth]{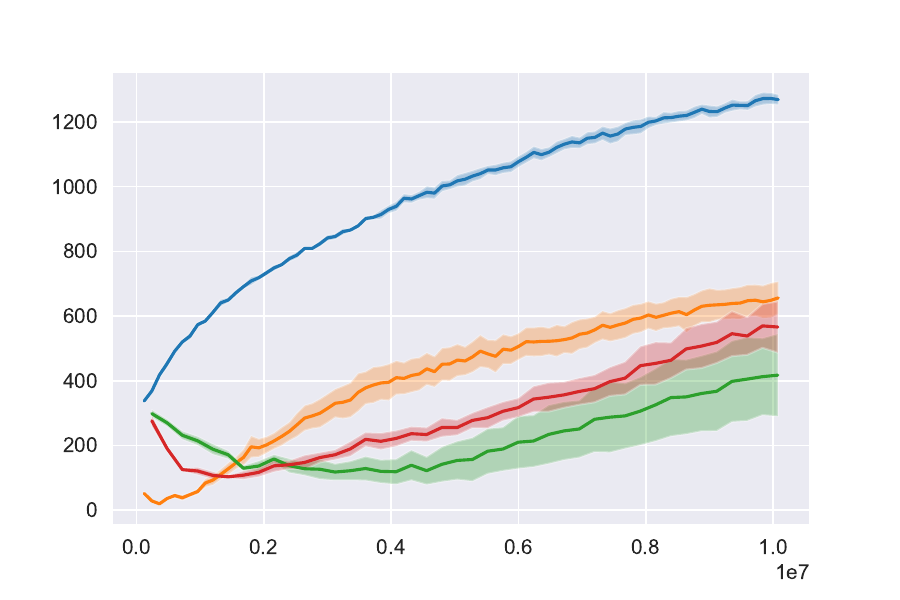}  
&\includegraphics[width=.15\linewidth]{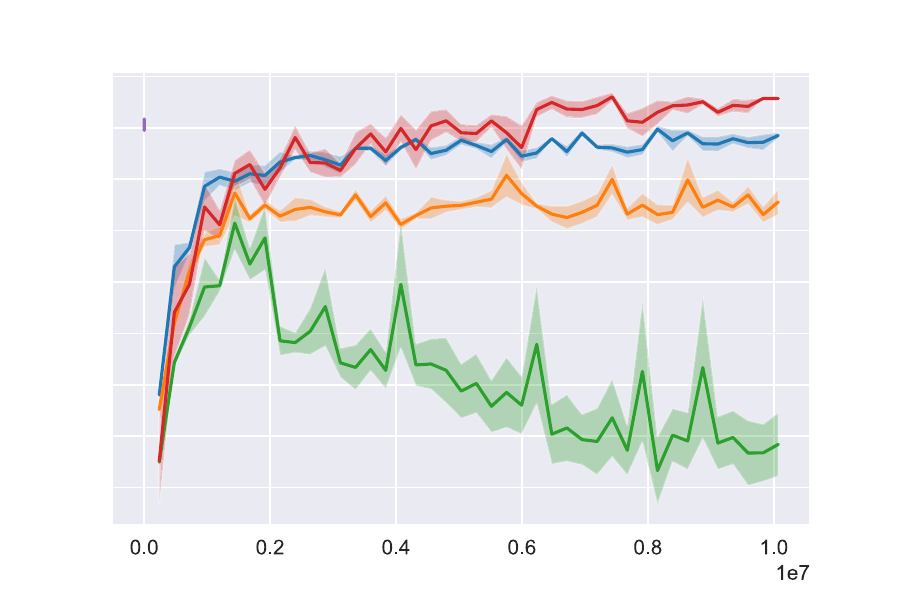}  \\
  \bottomrule
\end{tabular}
\caption{Welfare throughout training the benchmark algorithms (including a \emph{vanilla} implementation of contracting using off-the-shelf deep reinforcement learning algorithms). In simple static domains, contracting achieves welfare that is close to joint optimality, but more complex domains (i.e. Harvest and Cleanup), biased policy exploration due to the difficulty of learning best responses has a greater effect. Therefore, vanilla contracting suffers in performance. For each figure, the $x$-axis plots number of environment steps (e.g. all agents taking an action is one step), and error is one standard deviation over 5 independent runs.} 
\label{fig:linesmatrix}
\end{figure*}

\begin{figure*}[!ht]
\begin{tabular}{lcccccc}
\toprule
\# agents &Pris. Dilemma & Public Goods&Harvest &Cleanup &Merge \\
\midrule
\includegraphics[width=.14\linewidth]{figs/two.png}
 &\includegraphics[width=.14\linewidth]{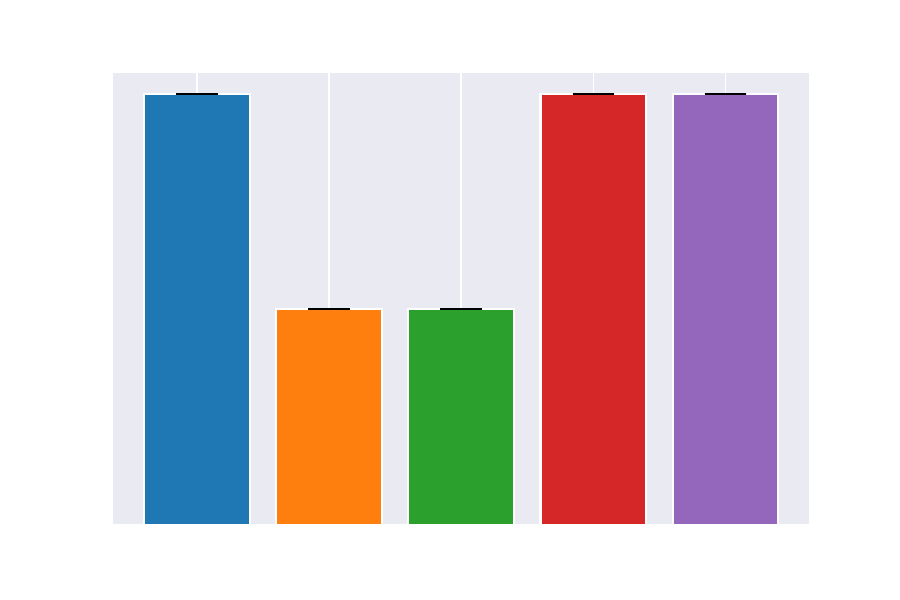} & \includegraphics[width=.14\linewidth]{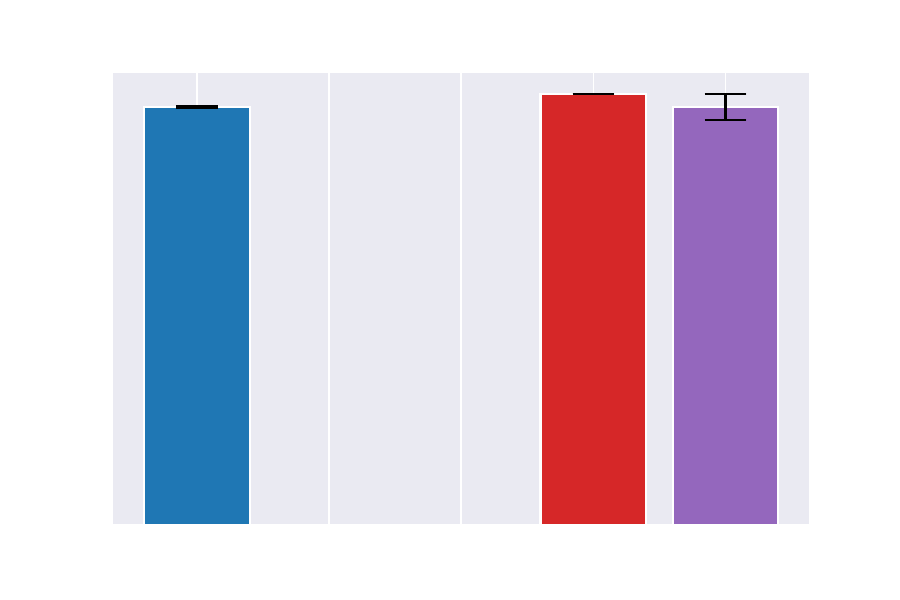}
&\includegraphics[width=.14\linewidth]{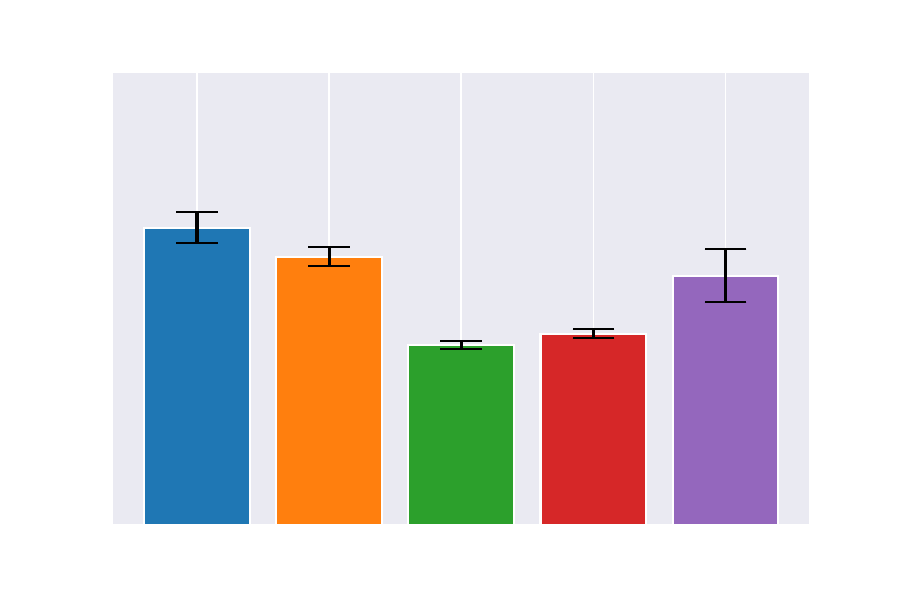}
&\includegraphics[width=.14\linewidth]{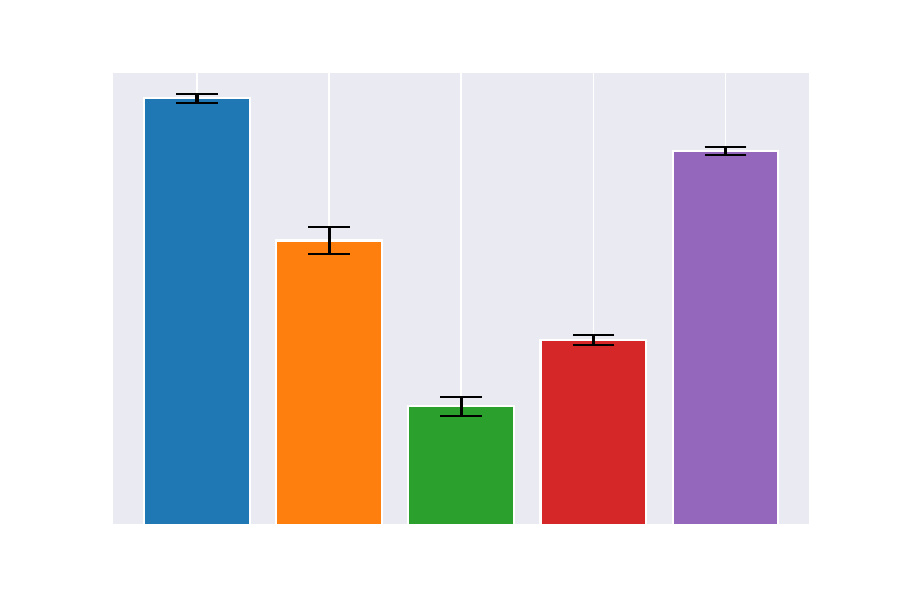} 
&\includegraphics[width=.14\linewidth]{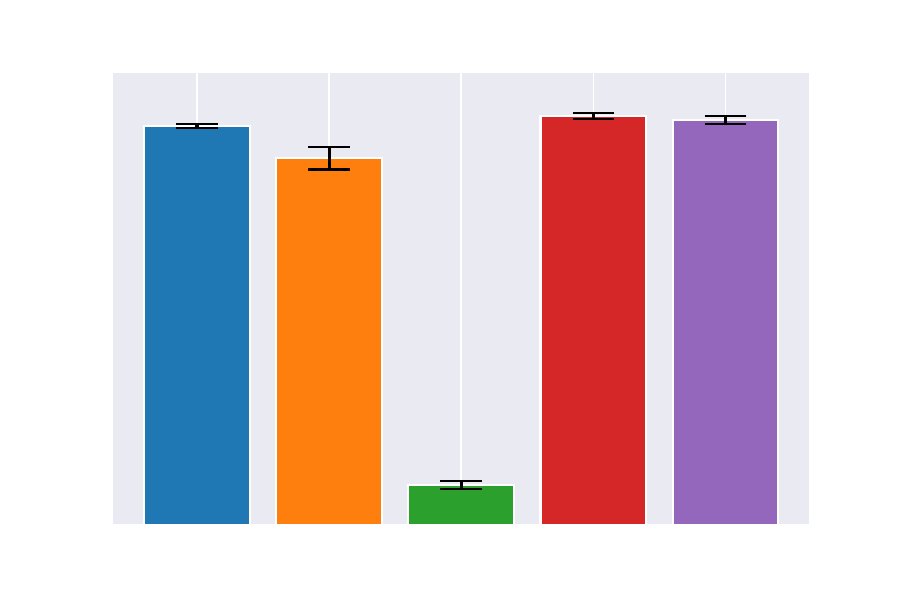} \\
\includegraphics[width=.14\linewidth]{figs/four.png} &\includegraphics[width=.14\linewidth]{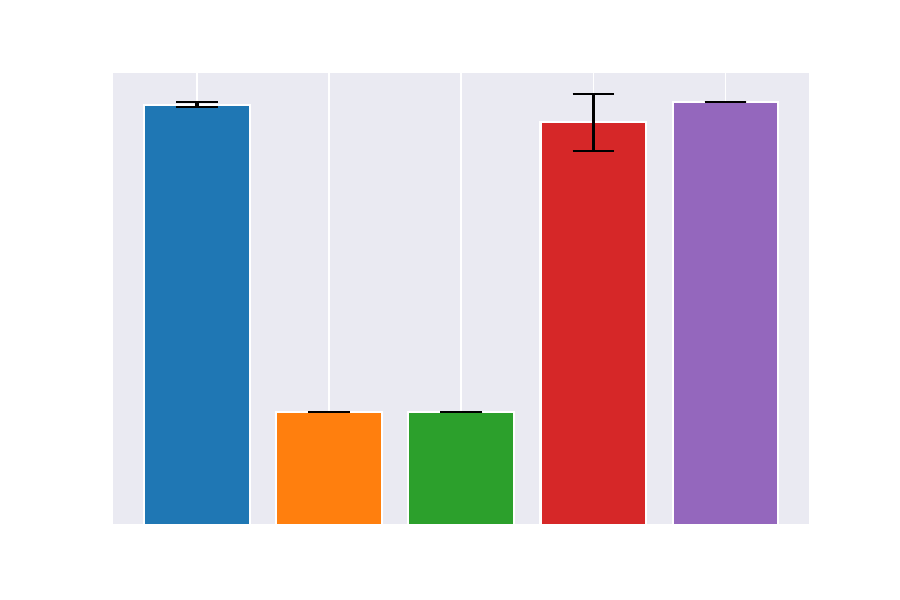} &\includegraphics[width=.14\linewidth]{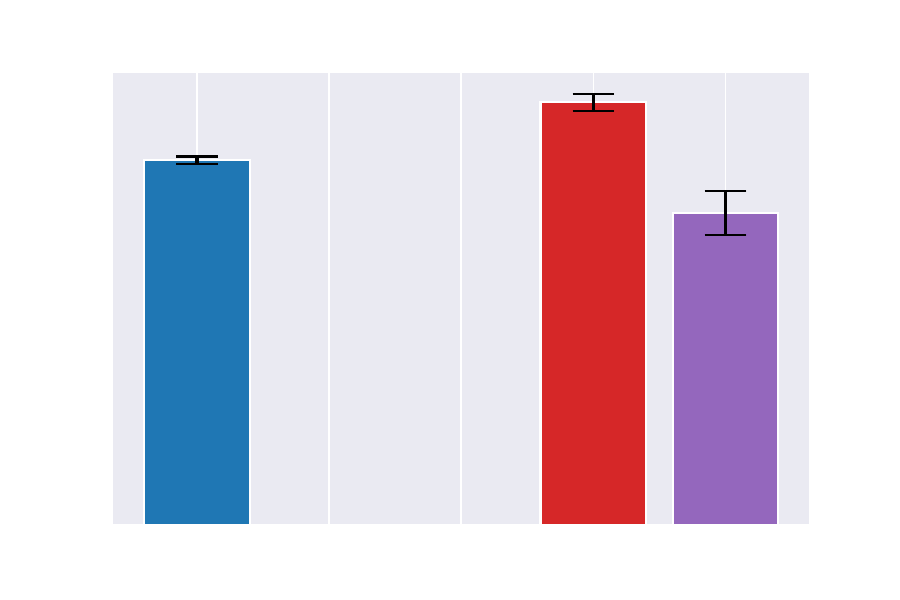}
&\includegraphics[width=.14\linewidth]{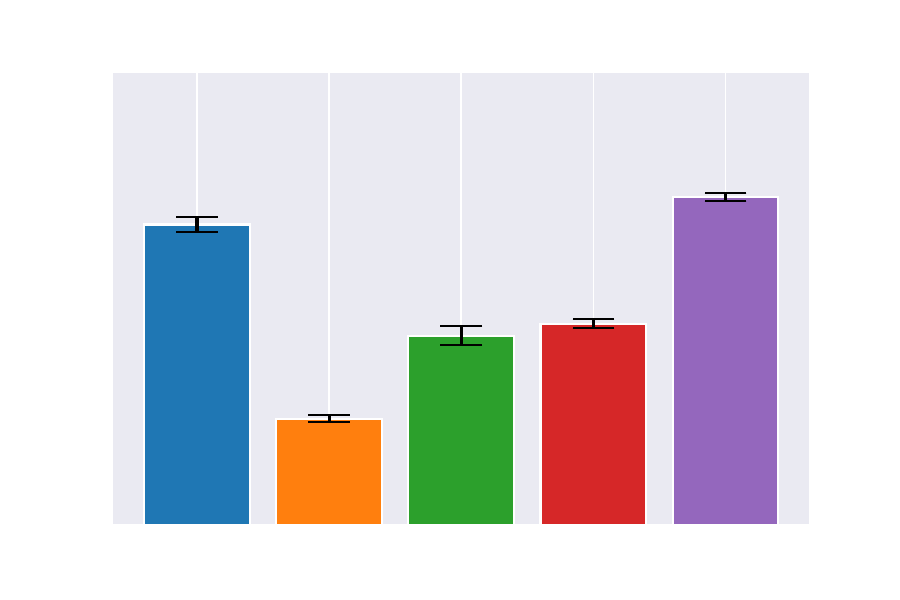}
&\includegraphics[width=.14\linewidth]{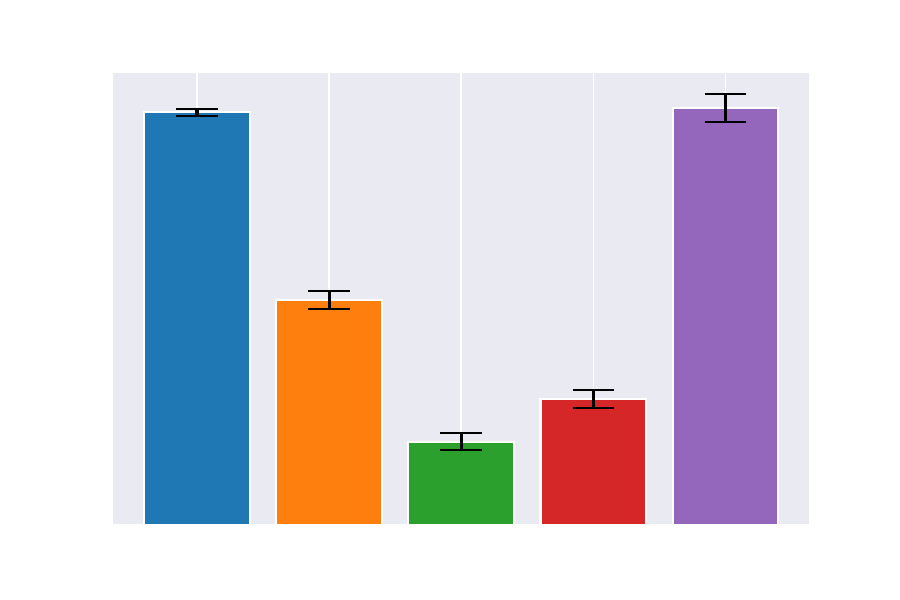}
&\includegraphics[width=.14\linewidth]{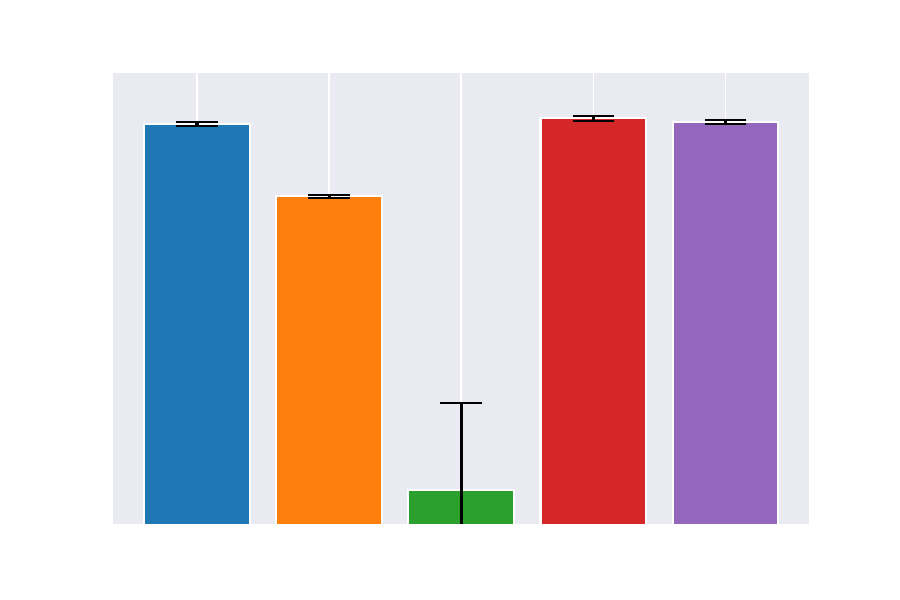} \\
 \includegraphics[width=.14\linewidth]{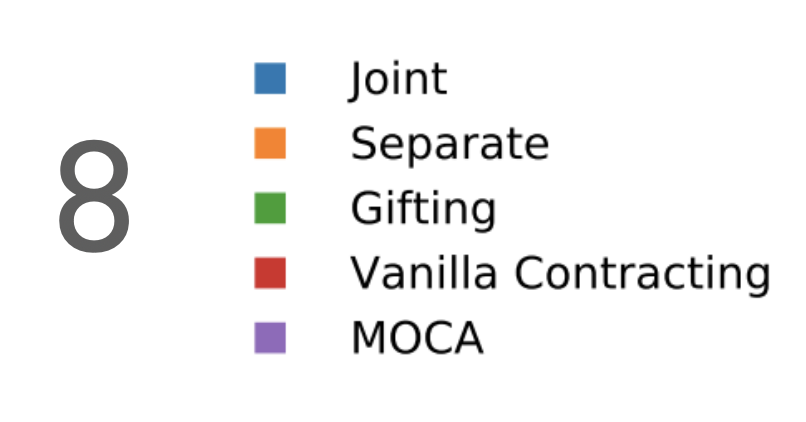}
 &\includegraphics[width=.14\linewidth]{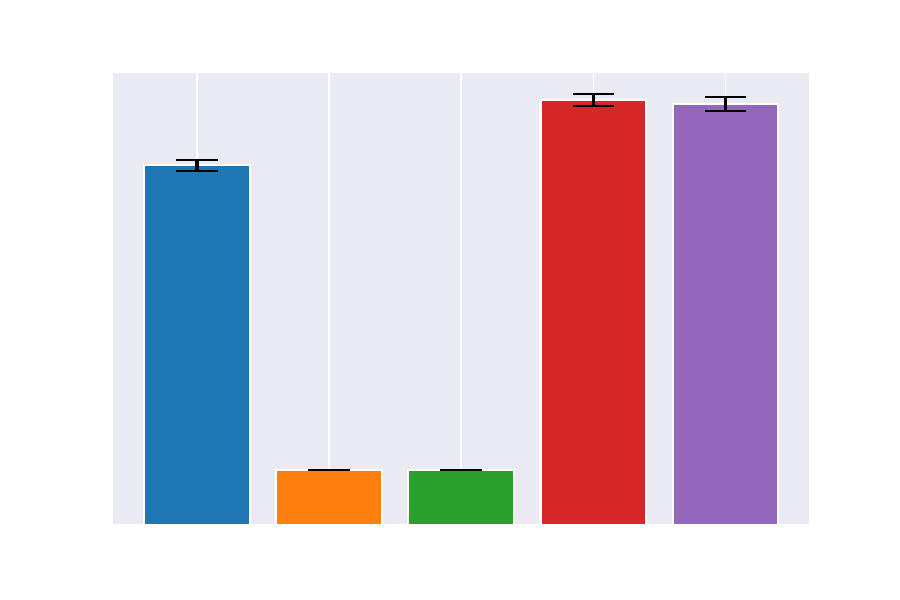} &\includegraphics[width=.14\linewidth]{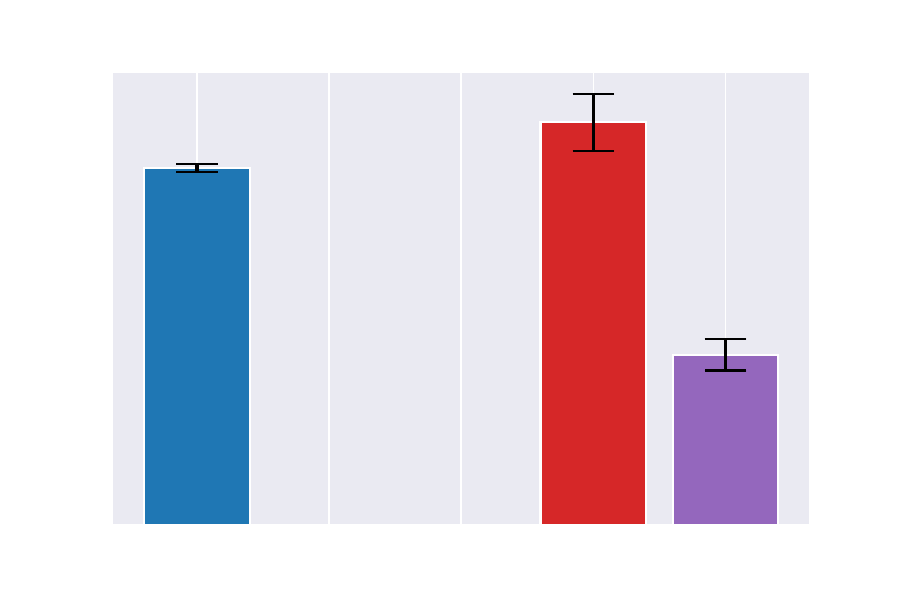}
  &\includegraphics[width=.14\linewidth]{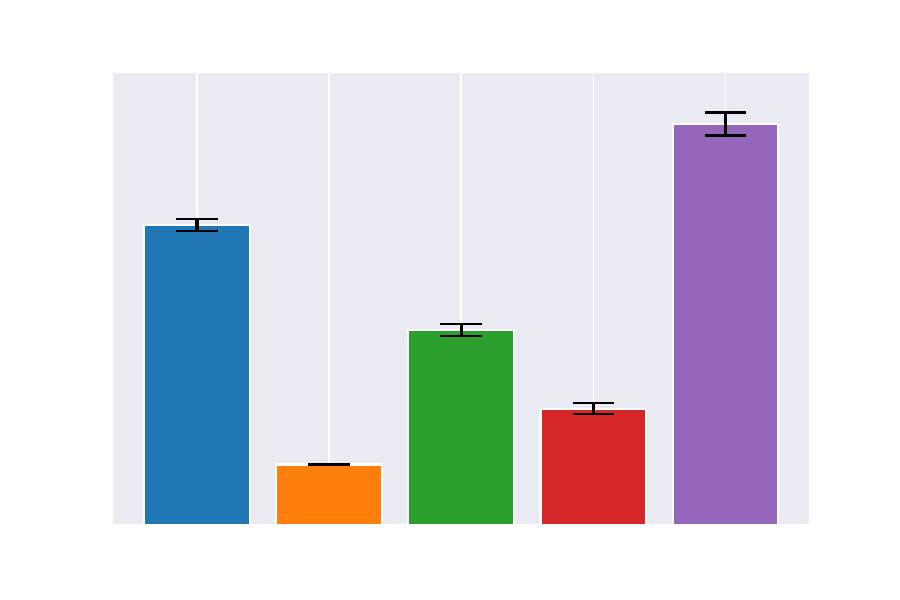}
  &\includegraphics[width=.14\linewidth]{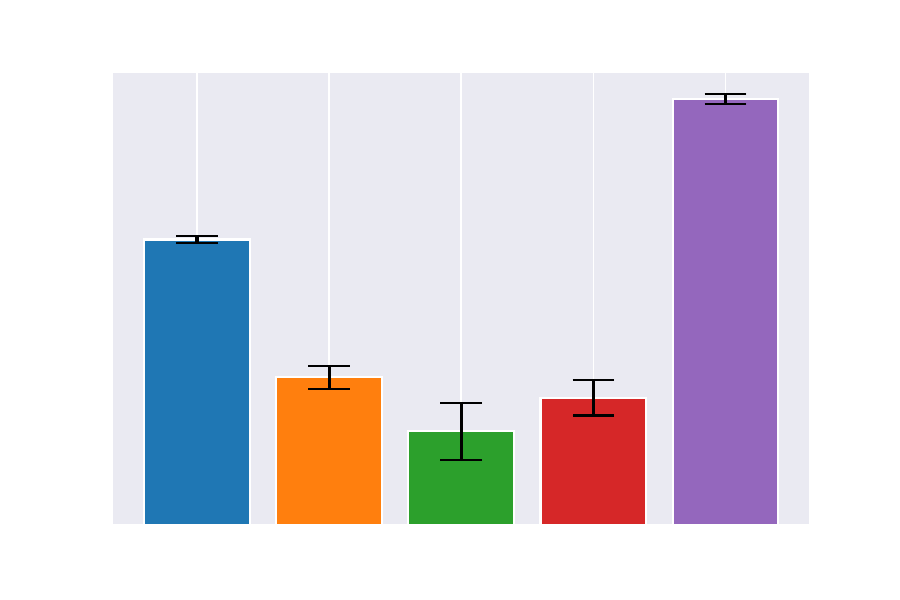}
  &\includegraphics[width=.14\linewidth]{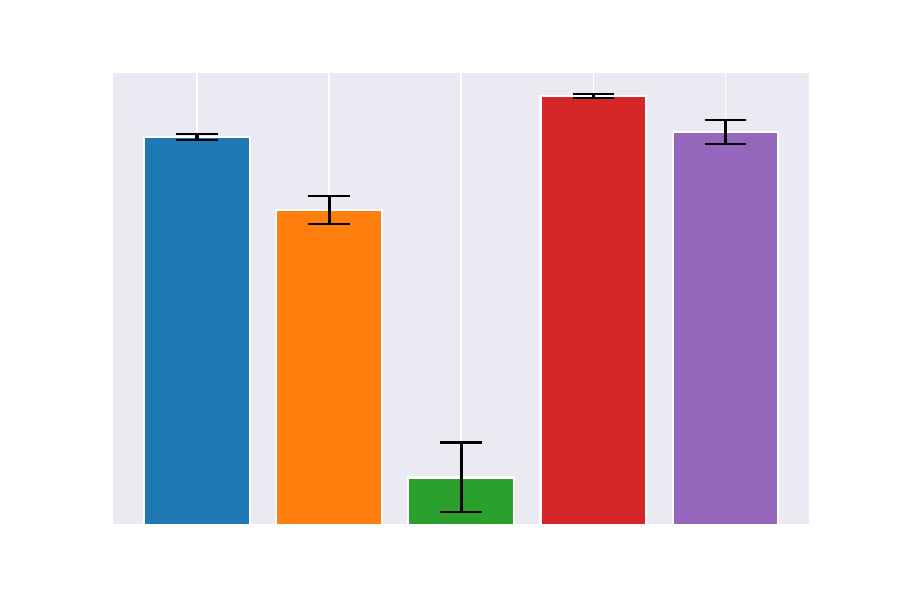} \\
  \bottomrule
\end{tabular}
\caption{Experimental results including MOCA. Every plot is the mean social reward per episode at the end of training (1M plays of the static domains, 5M timesteps for the dynamic domains) for each of the 5 algorithmic setups tested. Cells vary across number of agents present (2, 4, 8), environments (Prisoner's Dilemma, Public Goods, Harvest, Cleanup, Emergency Merge) with each cell comparing different algorithms (joint, gifting, separate, vanilla contracting, MOCA). Error bars denote one standard deviation over five independent runs. For simpler games in the left two columns, MOCA attains higher social reward than both gifting and separate training. However, it fails to match joint training in Public Goods, since this is a domain with simple environment dynamics where learning to respond to contracts is difficult, so early best responses are more likely to be informative, and biasing towards these early on is likely to help performance. For all of the more complex domains in the right 3 columns, contracting leads to higher social reward than gifting and separate training, and always at least matches that of joint training, except for 2-agent Harvest and Cleanup, where sufficient resources are available to make these very mild social dilemmas, as the higher welfare resulting from separate training compared to joint training shows. In Emergency Merge, both vanilla contracting and MOCA contracting significantly outperforms separate training. In several domains, contracting outperforms joint training, which is a result of the large action space of the joint problem.}
\label{fig:matrix}
\end{figure*}

In the more complex domains (right three columns of \Cref{fig:matrix}), MOCA in almost all cases attains at least the level of social welfare as joint training, and often far exceeds it. The exceptions to this trend are 2-agent Harvest and 2-agent Cleanup---in both cases, joint training and separate training exhibit strong performance (the former performing comparably to MOCA in both cases, the latter only for Harvest). The reason for this is that the two-agent settings for this problem are not strong social dilemmas, since the grid is wide enough that agents will not directly interact. Notably, this trend even applies in cases where the vanilla contracting fails (particularly in Harvest and Cleanup), motivating MOCA. In the merge domain, MOCA, contracting with a standard PPO training, and joint training, all perform similarly, given the fact that this is a substantially lower-dimensional, and simpler in terms of best-response policies, than Harvest or Cleanup.

\subsection{Experiments on Monotonicity} 
From our theory, we would predict that the average welfare attained by contracting, if agents have learned to enact subgame perfect equilibrium in the contracting game, should decrease as we inject noise into the contracting process. Indeed, we find that this happens, see \Cref{subfig:monotonic_reward}. Specifically, we find that, in cleanup with 4 agents, as we increase the probability that agents cleaning the river go unnoticed in the contract, the lower that social welfare becomes when learned with MOCA. However, we note that the monotonicity we see in this dynamic setup is not uniform: we see a significant drop when adding any noise at all, then a relative plateau for middle values of $\alpha$, and a sharp drop when going from $\alpha=0.75$ to $\alpha=1.0$. However, our theory doesn't say anything about which contracts are actually selected in each case, and as we see in \Cref{subfig:monotonic_params}, no strong pattern emerges between  $\alpha$ and contract parameter. Moreover, there seems to be a strong amount of variance regarding the precise selected contract at all levels. 

\begin{figure}[t!]
\begin{subfigure}[b]{0.45\linewidth}
    \centering
    \includegraphics[width=\linewidth]{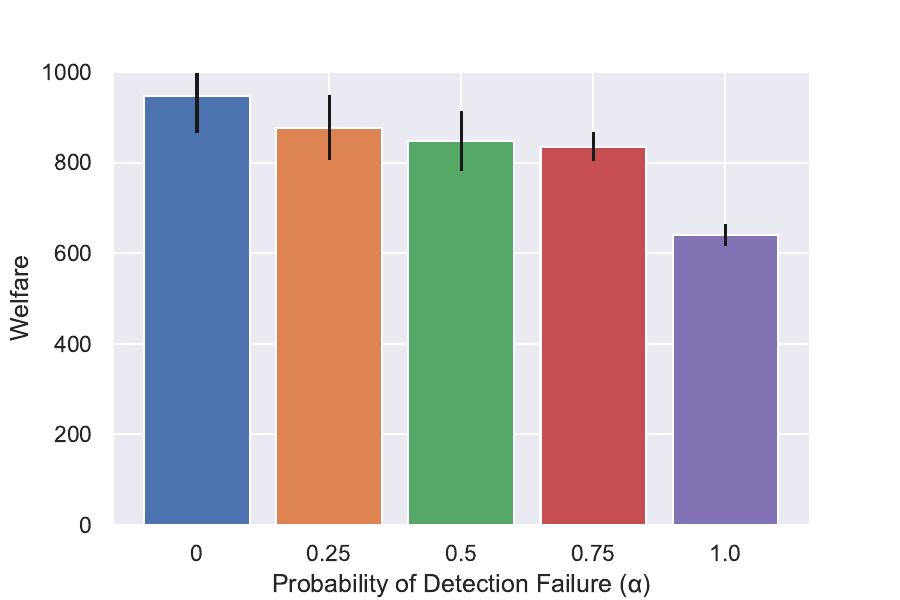}
    \caption{Attained welfare}
    \label{subfig:monotonic_reward}
\end{subfigure}
\begin{subfigure}[b]{0.45\linewidth}
    \centering
    \includegraphics[width=\linewidth]{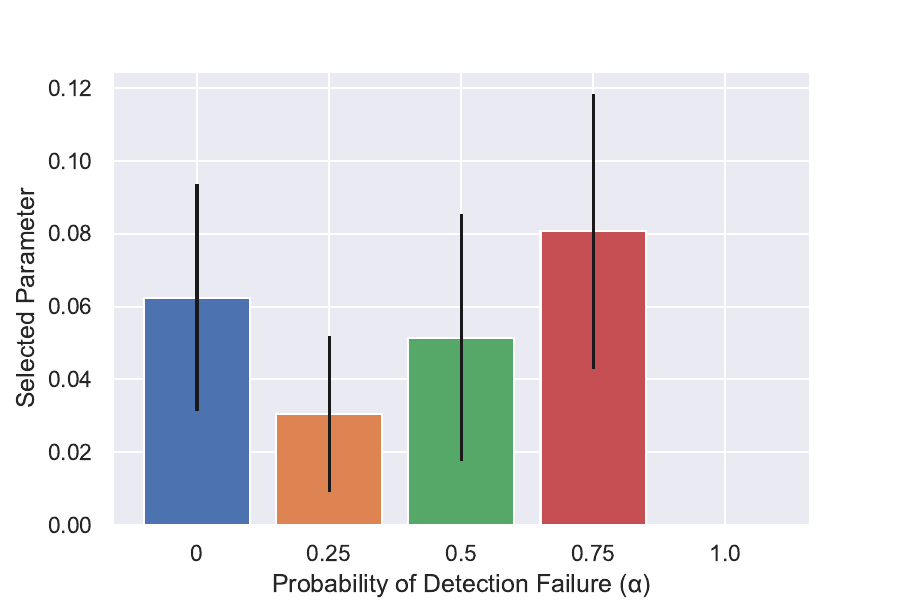}
    \caption{Negotiated contract parameters}
    \label{subfig:monotonic_params}
\end{subfigure}
\caption{Experimental validation of contract feature monotonicity. To evalutate the theory of \Cref{sec:features}, we study the performance of MOCA contracting after introducing a probability $\alpha$ of failing to detect contractible behaviour. This frames the detection of contractible behaviour as a \textit{random feature}, and our contracts can be understood as contracting on such random features. Viewed in this way, we can predict from \Cref{thm:randomfeat} that welfare should decrease with increasing $\alpha$, which is exactly what we observe in \Cref{subfig:monotonic_reward}. The decay, however, is far from linear, as we see a plateau for middle values of $\alpha$, and a sharp drop from $\alpha=0.75$ to $\alpha=1.0$. Moreover, our theory does not predict anything about which parameter is selected at equilibrium, and none in particular emerges from this experiment, as seen in \Cref{subfig:monotonic_params}. Note that $\alpha=0$ is completely equivalent to standard MOCA contracting, and $\alpha=1$ is completely equivalent to totally separate training, without contracts. This experiment was run on Cleanup with 4 agents, with the contractible behaviour being the agent cleaning the river, and was averaged over 5 seeds, the error plotted being one standard deviation.}

\label{fig:monotonic}
\end{figure}

\section{Related Work}\label{sec:related} % TODO: ANDY ADD JOURNAL ARTICLES INTO RELATED WORK kimplementation and duetting
We review related work in Computer Science and Economics. 

Our work intends to mitigate social dilemmas in games. In addition to classical static social dilemmas such as Prisoner's Dilemma (compare \cite{tucker1983mathematics}), a public goods game (compare \cite{janssen2003adaptation}) and Stag Hunt \cite{rousseau1985discourse}, we also consider more complex social dilemmas such as the Harvest and Cleanup domains of \cite{hughes2018inequity}. Cooperation and prosociality in complex domains are of keen interest to MARL researchers, with related challenges including dilemmas like Gathering and Wolfpack \cite{leibo2017multi}, the StarCraft challenge \cite{samvelyan2019starcraft}, or more recently with MARL results in Diplomacy \cite{kramar2022negotiation}.

We relate to the study of augmentations of Markov games to enhance pro-social behavior. Gifting \cite{lupu2020gifting,wang2021emergent} is as augmentation expanding agents' action spaces to allow for reward transfers to other agents. \cite{wang2021emergent} prove that gifting is unable to change the set of Nash equilibria of the underlying game. Our approach differs from gifting in that contracting forces commitment to a given modification of reward \emph{before} taking an action in the original game, and that this commitment is binding for the length of time the contract is in force. This improves total welfare by allowing incentivize play that is not a Nash equilibrium in the original game. Relatedly, (binding) contracts in the sense of \citet{hughes2020learning} are contracts that force certain actions and do not allow for reward transfers. They will only ever be enacted when all agents are made better off when a contract is accepted, greatly limiting the effectiveness of the system design. For example, no agent will ever consent to a binding contract cleaning trash in Cleanup. Further, since binding contracts are action-level contracts, the system designer would have to manually encode a \enquote{clean trash} policy, instead of relying on a reward signal to \emph{incentivise} them to clean.

A large class of prior work has considered different forms of agent commitments. \cite{smith1980contract} proposes the contract net protocol, which allocates tasks among agents and commits them to complete a particular task. \cite{hughes2020learning} is a recent study in this line of work. The paper considers multi-agent zero-sum games in which agents may give other agents the option to commit to taking a particular action. \cite{hughes2020learning}'s contracts---which we will call \emph{binding contracts} as opposed to our concept of \emph{formal contracts}---might be insufficient to induce the desired joint behavior, as agents will only commit to actions which improve their own individual welfare over the basic game, leading to Pareto optima, which may not be jointly optimal. \cite{Han2017} considers the idea of commitments, without transfers, in evolutionary game theory. The paper \cite{mcaleer2021improving} lets an auxiliary agent propose a Pareto-optimal equilibrium in a game. \cite{sandholm1996advantages} proposes to allow agents to be able to decommit from a task and paying a side payment. Our approach can be seen as \enquote{soft commitment} in which agents always only incur a cost in terms of reward when taking different actions, but are not forced to take a particular action. \cite{sodomka2013coco} considers in 2-player games the proposal of commitment and side payments, and reaches social cooperation. We provide a general approach that only considers reward transfer without the need to commit to actions. The idea of negotiations between rounds to arrive to a commitment to an action, was considered in \cite{de2020strategic}. Formal contracting does not have the dynamic structure of a negotiation and lets a proposing agent make a take-it-or-leave-it offer. Other related publications study the emergence and learning of social norms \cite{vinitsky2021learning,koster2022spurious} as well as conventions \cite{koster2020model}. In contrast to the present paper, these do not require explicit consent by agents.

Another related paradigm to ours is Stackelberg Learning. In such models, typically, a special agent, the principal, optimizes incentives for other agents in a bi-level optimization problem. Stackelberg learning has received a lot of attention in strategic Machine Learning \cite{hardt2016strategic,milli2019social,zrnic2021leads} and has been used to learn large scale mechanisms such as auctions \cite{duetting,curry2022differentiable}. Also, \cite{yang2020learning}'s approach to learning to incentivize other learning agents may be seen as a Stackelberg Learning. \cite{kimplementation} considers a mediator's problem in Normal-Form Games. In contrast to Stackelberg Learning, the focus of formal contracts is that no additional agents---Stackelberg leaders---are introduced into an environment, but proposing agents are part of the environment.

Formal contracts have been considered as an alternative to relational contracts in the fields of organizational Economics, see \cite{gil2017formal} and \cite[5.2.3]{gibbons2005four}. The setting of an agent proposing state-dependent reward transfers has received considerable attention in contract economics, compare the literature following \cite{mussa1978monopoly}, and mechanism design, compare \cite{hurwicz1973design,vickrey1961counterspeculation,myerson1981optimal}. In our proof of \Cref{thm:main}, we use a \emph{forcing contract}, compare \cite{holmstrom1979moral}.
\section{Discussion}\label{sec:conclusion}
We discuss that the assumption that a single agent proposes a contract is crucial for our results, and its fairness implications, in \Cref{subsec:limits}. Finally, we discuss approaches to scaling formal contracting to more complex domains in \Cref{subsec:scaling}. We discuss fairness concerns in contracting in \Cref{subsec:fairness}.
\subsection{Limitations for Formal Contracting}\label{subsec:limits}
One crucial assumption in our analysis is that a single agent proposes contracts. Game-theoretic analysis, given in \Cref{apx:proofs}, shows that if two or more agents may propose in a game, SPEs may be socially suboptimal. The intuition is that an agent $i$ may choose a contract to affect the state distribution in a way that gives the agent a rejection reward when $j$ proposes a contract, hence increasing their reward when contracting. Our game-theoretic analysis also showed that unfair outcomes might result from contracts. As hence proposal by different agents and joint optimal behavior are incompatible, system designers who would like to ensure fairness need to design contracts in a way that limit the number of transfers that can be made, potentially at the expense of welfare.

\subsection{Scaling Formal Contracting}\label{subsec:scaling}
The clearest avenue for future work is in scaling contracts to more realistic domains. Here, we outline three ways to do that.

First, contracts in this work were hand-engineered with relevant internal logic, in order to make the transfers a useful signal. For this approach to scale, a complexity tradeoff must be managed: Contracts need to be flexible enough to extract enough relevant information to incentivize welfare-optimal play, while being simple enough for MARL agents to allow fast learning of which contracts to choose resp. accept. General techniques allowing to choose contracts would greatly improve the scalability of the method.

Manually managing this tradeoff is undesirable. In particular, there may be domains whose social dilemmas are not transparent, or it might be hard to design concise yet powerful contractible observation models. Therefore, learning which aspects of a state are useful for contracting will allow us to scale the approach to more realistic scenarios while keeping contract space sizes low.
     
Even with a fixed contract space, sample efficiency may be improved. MOCA took a first step into improving contract learning, by decreasing the bias in estimated $V_i^{\boldsymbol\pi} (s_0, \boldsymbol\theta)$ values. MOCA outperformed benchmarks in all, even complex, dynamic, environments that exhibited a social dilemma. Increasing sample efficiency will allow using formal contracting to mitigate social dilemmas in even more complex domains. One potential way to increase the sample efficiency of the first phase of MOCA is to leverage more of the literature on multi-task reinforcement learning methods \cite{devin2017learning, rusu_policy_2015, parisotto_actor-mimic_2015, teh2017distral}.

\subsection{Fairness}\label{subsec:fairness}
One striking observation in the proof of \Cref{thm:autspace} is that agents $j \in [n] \setminus \{i\}$ are indifferent between accepting the contract $\boldsymbol\theta^*$ and not accepting it. Hence, the contract may lead to an improvement in welfare relative to $\bpi$, but no agents except agent $i=1$ gets any benefit from this improvement. In many decentralized learning tasks, this is not of concern. For example, if a robot fleet needs to coordinate which locations to service, low reward for some of the agents is not of concern, as long as system performance (that is, welfare) is high. In others, this behavior may be unfair. Trading off fairness and welfare is an avenue for future work.
\bibliography{refs}

\begin{thebibliography}{44}
\providecommand{\natexlab}[1]{#1}
\providecommand{\url}[1]{{#1}}
\providecommand{\urlprefix}{URL }
\providecommand{\doi}[1]{\url{https://doi.org/#1}}
\providecommand{\eprint}[2][]{\url{#2}}
 \bibcommenthead

\bibitem[{Andrychowicz et~al(2017)Andrychowicz, Wolski, Ray, Schneider, Fong, Welinder, McGrew, Tobin, Pieter~Abbeel, and Zaremba}]{andrychowicz_hindsight_2017}
Andrychowicz M, Wolski F, Ray A, et~al (2017) Hindsight experience replay. Advances in neural information processing systems 30

\bibitem[{Bernstein et~al(2002)Bernstein, Givan, Immerman, and Zilberstein}]{bernstein2002complexity}
Bernstein DS, Givan R, Immerman N, et~al (2002) The complexity of decentralized control of markov decision processes. Mathematics of operations research 27(4):819--840

\bibitem[{Curry et~al(2022)Curry, Sandholm, and Dickerson}]{curry2022differentiable}
Curry M, Sandholm T, Dickerson J (2022) Differentiable economics for randomized affine maximizer auctions. arXiv preprint arXiv:220202872

\bibitem[{De~Jonge and Zhang(2020)}]{de2020strategic}
De~Jonge D, Zhang D (2020) Strategic negotiations for extensive-form games. Autonomous Agents and Multi-Agent Systems 34(1):1--41

\bibitem[{Devin et~al(2017)Devin, Gupta, Darrell, Abbeel, and Levine}]{devin2017learning}
Devin C, Gupta A, Darrell T, et~al (2017) Learning modular neural network policies for multi-task and multi-robot transfer. In: 2017 IEEE international conference on robotics and automation (ICRA), IEEE, pp 2169--2176

\bibitem[{D\"{u}tting et~al(2023)D\"{u}tting, Feng, Narasimhan, Parkes, and Ravindranath}]{duetting}
D\"{u}tting P, Feng Z, Narasimhan H, et~al (2023) Optimal auctions through deep learning: Advances in differentiable economics. J ACM \doi{10.1145/3630749}, \urlprefix\url{https://doi.org/10.1145/3630749}

\bibitem[{Fudenberg and Maskin(1986)}]{folktheorem}
Fudenberg D, Maskin E (1986) The folk theorem in repeated games with discounting or with incomplete information. Econometrica 54(3):533--554. \urlprefix\url{http://www.jstor.org/stable/1911307}

\bibitem[{Gibbons(2005)}]{gibbons2005four}
Gibbons R (2005) Four formal (izable) theories of the firm? Journal of Economic Behavior \& Organization 58(2):200--245

\bibitem[{Gibbons(1992)}]{gibbons1992game}
Gibbons RS (1992) Game theory for applied economists. Princeton University Press

\bibitem[{Gil and Zanarone(2017)}]{gil2017formal}
Gil R, Zanarone G (2017) Formal and informal contracting: Theory and evidence. Annual Review of Law and Social Science 13:141--159

\bibitem[{Han et~al(2017)Han, Pereira, and Lenaerts}]{Han2017}
Han TA, Pereira LM, Lenaerts T (2017) Evolution of commitment and level of participation in public goods games. Autonomous Agents and Multi-Agent Systems 31(3):561--583. \doi{10.1007/s10458-016-9338-4}, \urlprefix\url{https://doi.org/10.1007/s10458-016-9338-4}

\bibitem[{Hardt et~al(2016)Hardt, Megiddo, Papadimitriou, and Wootters}]{hardt2016strategic}
Hardt M, Megiddo N, Papadimitriou C, et~al (2016) Strategic classification. In: Proceedings of the 2016 ACM conference on innovations in theoretical computer science, pp 111--122

\bibitem[{Holmstr{\"o}m(1979)}]{holmstrom1979moral}
Holmstr{\"o}m B (1979) Moral hazard and observability. The Bell journal of economics pp 74--91

\bibitem[{Hughes et~al(2018)Hughes, Leibo, Phillips, Tuyls, Du{\'e}{\~n}ez-Guzm{\'a}n, Casta{\~n}eda, Dunning, Zhu, McKee, Koster et~al}]{hughes2018inequity}
Hughes E, Leibo JZ, Phillips MG, et~al (2018) Inequity aversion improves cooperation in intertemporal social dilemmas. arXiv preprint arXiv:180308884

\bibitem[{Hughes et~al(2020)Hughes, Anthony, Eccles, Leibo, Balduzzi, and Bachrach}]{hughes2020learning}
Hughes E, Anthony TW, Eccles T, et~al (2020) Learning to resolve alliance dilemmas in many-player zero-sum games. In: Proceedings of the 19th International Conference on Autonomous Agents and MultiAgent Systems, pp 538--547

\bibitem[{Hurwicz(1973)}]{hurwicz1973design}
Hurwicz L (1973) The design of mechanisms for resource allocation. The American Economic Review 63(2):1--30

\bibitem[{Janssen and Ahn(2003)}]{janssen2003adaptation}
Janssen M, Ahn T (2003) Adaptation vs. anticipation in public-good games. In: annual meeting of the American Political Science Association, Philadelphia, PA

\bibitem[{K{\"o}ster et~al(2020)K{\"o}ster, McKee, Everett, Weidinger, Isaac, Hughes, Du{\'e}{\~n}ez-Guzm{\'a}n, Graepel, Botvinick, and Leibo}]{koster2020model}
K{\"o}ster R, McKee KR, Everett R, et~al (2020) Model-free conventions in multi-agent reinforcement learning with heterogeneous preferences. arXiv preprint arXiv:201009054

\bibitem[{K{\"o}ster et~al(2022)K{\"o}ster, Hadfield-Menell, Everett, Weidinger, Hadfield, and Leibo}]{koster2022spurious}
K{\"o}ster R, Hadfield-Menell D, Everett R, et~al (2022) Spurious normativity enhances learning of compliance and enforcement behavior in artificial agents. Proceedings of the National Academy of Sciences 119(3)

\bibitem[{Kram{\'a}r et~al(2022)Kram{\'a}r, Eccles, Gemp, Tacchetti, McKee, Malinowski, Graepel, and Bachrach}]{kramar2022negotiation}
Kram{\'a}r J, Eccles T, Gemp I, et~al (2022) Negotiation and honesty in artificial intelligence methods for the board game of diplomacy. Nature Communications 13(1):7214

\bibitem[{Leibo et~al(2017)Leibo, Zambaldi, Lanctot, Marecki, and Graepel}]{leibo2017multi}
Leibo JZ, Zambaldi V, Lanctot M, et~al (2017) Multi-agent reinforcement learning in sequential social dilemmas. arXiv preprint arXiv:170203037

\bibitem[{Liang et~al(2018)Liang, Liaw, Nishihara, Moritz, Fox, Goldberg, Gonzalez, Jordan, and Stoica}]{liang2018rllib}
Liang E, Liaw R, Nishihara R, et~al (2018) Rllib: Abstractions for distributed reinforcement learning. In: International Conference on Machine Learning, PMLR, pp 3053--3062

\bibitem[{Lupu and Precup(2020)}]{lupu2020gifting}
Lupu A, Precup D (2020) Gifting in multi-agent reinforcement learning. In: Proceedings of the 19th International Conference on Autonomous Agents and MultiAgent Systems, pp 789--797

\bibitem[{McAleer et~al(2021)McAleer, Lanier, Dennis, Baldi, and Fox}]{mcaleer2021improving}
McAleer S, Lanier J, Dennis M, et~al (2021) Improving social welfare while preserving autonomy via a pareto mediator. arXiv preprint arXiv:210603927

\bibitem[{Milli et~al(2019)Milli, Miller, Dragan, and Hardt}]{milli2019social}
Milli S, Miller J, Dragan AD, et~al (2019) The social cost of strategic classification. In: Proceedings of the Conference on Fairness, Accountability, and Transparency, pp 230--239

\bibitem[{Monderer and Tennenholtz(2004)}]{kimplementation}
Monderer D, Tennenholtz M (2004) K-implementation. J Artif Int Res 21(1):37–62

\bibitem[{Mussa and Rosen(1978)}]{mussa1978monopoly}
Mussa M, Rosen S (1978) Monopoly and product quality. Journal of Economic theory 18(2):301--317

\bibitem[{Myerson(1981)}]{myerson1981optimal}
Myerson RB (1981) Optimal auction design. Mathematics of operations research 6(1):58--73

\bibitem[{Osborne and Rubinstein(1994)}]{osborne1994course}
Osborne MJ, Rubinstein A (1994) A course in game theory. MIT press

\bibitem[{Parisotto et~al(2015)Parisotto, Ba, and Salakhutdinov}]{parisotto_actor-mimic_2015}
Parisotto E, Ba JL, Salakhutdinov R (2015) Actor-mimic: {Deep} multitask and transfer reinforcement learning. arXiv preprint arXiv:151106342

\bibitem[{Rousseau(1985)}]{rousseau1985discourse}
Rousseau JJ (1985) A discourse on inequality. Penguin

\bibitem[{Rusu et~al(2015)Rusu, Colmenarejo, Gulcehre, Desjardins, Kirkpatrick, Pascanu, Mnih, Kavukcuoglu, and Hadsell}]{rusu_policy_2015}
Rusu AA, Colmenarejo SG, Gulcehre C, et~al (2015) Policy distillation. arXiv preprint arXiv:151106295

\bibitem[{Samvelyan et~al(2019)Samvelyan, Rashid, De~Witt, Farquhar, Nardelli, Rudner, Hung, Torr, Foerster, and Whiteson}]{samvelyan2019starcraft}
Samvelyan M, Rashid T, De~Witt CS, et~al (2019) The starcraft multi-agent challenge. arXiv preprint arXiv:190204043

\bibitem[{Sandholm and Lesser(1996)}]{sandholm1996advantages}
Sandholm TW, Lesser VR (1996) Advantages of a leveled commitment contracting protocol. In: AAAI/IAAI, Vol. 1, Citeseer, pp 126--133

\bibitem[{Schulman et~al(2017)Schulman, Wolski, Dhariwal, Radford, and Klimov}]{schulman2017proximal}
Schulman J, Wolski F, Dhariwal P, et~al (2017) Proximal policy optimization algorithms. arXiv preprint arXiv:170706347

\bibitem[{Smith(1980)}]{smith1980contract}
Smith RG (1980) The contract net protocol: High-level communication and control in a distributed problem solver. IEEE Transactions on computers 29(12):1104--1113

\bibitem[{Sodomka et~al(2013)Sodomka, Hilliard, Littman, and Greenwald}]{sodomka2013coco}
Sodomka E, Hilliard E, Littman M, et~al (2013) Coco-q: Learning in stochastic games with side payments. In: International Conference on Machine Learning, PMLR, pp 1471--1479

\bibitem[{Teh et~al(2017)Teh, Bapst, Czarnecki, Quan, Kirkpatrick, Hadsell, Heess, and Pascanu}]{teh2017distral}
Teh Y, Bapst V, Czarnecki WM, et~al (2017) Distral: Robust multitask reinforcement learning. Advances in Neural Information Processing Systems 30

\bibitem[{Tucker and Straffin~Jr(1983)}]{tucker1983mathematics}
Tucker AW, Straffin~Jr PD (1983) The mathematics of tucker: A sampler. The Two-Year College Mathematics Journal 14(3):228--232

\bibitem[{Vickrey(1961)}]{vickrey1961counterspeculation}
Vickrey W (1961) Counterspeculation, auctions, and competitive sealed tenders. The Journal of finance 16(1):8--37

\bibitem[{Vinitsky et~al(2021)Vinitsky, K{\"o}ster, Agapiou, Du{\'e}{\~n}ez-Guzm{\'a}n, Vezhnevets, and Leibo}]{vinitsky2021learning}
Vinitsky E, K{\"o}ster R, Agapiou JP, et~al (2021) A learning agent that acquires social norms from public sanctions in decentralized multi-agent settings. arXiv preprint arXiv:210609012

\bibitem[{Wang et~al(2021)Wang, Beliaev, B{\i}y{\i}k, Lazar, Pedarsani, and Sadigh}]{wang2021emergent}
Wang WZ, Beliaev M, B{\i}y{\i}k E, et~al (2021) Emergent prosociality in multi-agent games through gifting. arXiv preprint arXiv:210506593

\bibitem[{Yang et~al(2020)Yang, Li, Farajtabar, Sunehag, Hughes, and Zha}]{yang2020learning}
Yang J, Li A, Farajtabar M, et~al (2020) Learning to incentivize other learning agents. Advances in Neural Information Processing Systems 33:15208--15219

\bibitem[{Zrnic et~al(2021)Zrnic, Mazumdar, Sastry, and Jordan}]{zrnic2021leads}
Zrnic T, Mazumdar E, Sastry S, et~al (2021) Who leads and who follows in strategic classification? Advances in Neural Information Processing Systems 34:15257--15269

\end{thebibliography}
\appendix

\clearpage 
\section{Omitted Proofs}\label{apx:proofs}
We provide the proof of \Cref{thm:detecatbledev}, the result \Cref{thm:main} is a direct consequence.
\begin{proof}[Proof of \Cref{thm:detecatbledev}]
Consider any ${\bTheta} \supseteq {\bTheta}_{\operatorname{\operatorname{full}}}$. Fix an arbitrary socially optimal policy profile $\bpi^*$, and consider a forcing contract ${\btheta}^*$ which punishes any deviation from $\bpi^*$ by $-\frac{R_{\text{max}}}{\gamma^T(1-\gamma)}$, where $T$ is the time that a deviation is detected. Let the signing transfer $\bpi^*$ be such that all payoff gain for a non-proposer due to $\btheta^*$ over no contract is transferred to the proposer. Note that the punishment is a larger negative reward than any reward that can be achieved in the game from the outset. Hence, after acceptance of the contract, $\bpi^*$ constitutes an SPE. By assumption, after rejection of the contract, it also constitutes an SPE. By definition, all agents but the proposer are indifferent between accepting and rejecting the contract. In particular, accepting it is an SPE. It remains to show that the proposer cannot get more payoff than under $\btheta^*$. All other agents need to get at least $V_i^{M, \bpi}(s_0)$ to accept a contract, it is impossible for the proposer, in any contract following a contract that is accepted, to get more welfare than 
\[
W^{\bpi^*}(s_0) - \sum_{i=2}^n V_i^{M, \bpi_0}(s_0).
\]
This is exactly like the payoff that $\btheta^*$ yields.

It remains to show that all SPE must lead to a jointly optimal policy profile in $M+\btheta$ for the contract proposed by $\btheta$. First observe that it cannot be optimal for the proposer to propose a contract that wouldn't be accepted (proposing $\btheta^*$ would yield higher payoff). Hence, a contract that is accepted is proposed. Note that if a jointly sub-optimal policy profile $\bpi'$ would be played following the proposed contract $\btheta$, the proposer payoff would be bounded by 
\[
W^{\bpi'}(s_0) - \sum_{i=2}^n V_i^{M, \bpi_0}(s_0) < W^{\bpi^*}(s_0) - \sum_{i=2}^n V_i^{M, \bpi_0}(s_0),
\]
which is lower than what the proposer would get under $\btheta^*$.
\end{proof}

\begin{proof}[Proof of \Cref{thm:autspace}]
As in the proof sketch, we only consider the first inequality in depth here only: the upper inequality follows by exchanging WCSPW for BCSPW, \enquote{worst case subgame-perfect equilibrium} with \enquote{best case subgame-perfect equilibrium}, changing $\argmin$ into $\argmax$, and exchanging inequalities in the second part of the proof and the section headings. 

We first show the second claim -- that is, assuming $\theta^*$ and $\pi^*$ are attained, we wish to construct an SPE attaining $W(\mathcal{C})$. Fix a Markov game $M$, a contracting space ${\bTheta}$ that admits arbitrary unconditional transfers to augment it and an arbitrary observation model $O_0$, giving contracting model $\mathcal{C}$. As the supremum in \eqref{eq:attainmentofsupremal} is attained, there is an SPE attaining ${\btheta}^*$ in $M^{\mathcal C}$. Let ${\btheta}^* \in \argmax_{{\btheta} \in {\bTheta}} \operatorname{WCSPW}({\btheta})$, and
\[
\bpi^* \in \argmin_{\bpi \in \operatorname{SPE}(M+{\btheta}^*)} W^{\bpi}(s_0)
\]
Then, take
\[
\bpi \in \argmin_{\bpi \in \operatorname{SPE}(M)} W^{\bpi}(s_0)
\]
Moreover, define the following vector
\[
\mathbf v = \begin{pmatrix}
    \sum_{i=2}^N V_i^{M+{\btheta}^*, \bpi^*}(s_0) - V_i^{M, \bpi}(s_0) \\
    V_2^{M, \bpi}(s_0) - V_2^{M+{\btheta}^*, \bpi^*}(s_0) \\
    \vdots \\
    V_N^{M, \bpi}(s_0) - V_N^{M+{\btheta}^*, \bpi^*}(s_0)
\end{pmatrix}
\]
and the following contract
\[
\theta^*_{\text{tr}}(o) = \begin{cases}
   \mathbf{v} & o = \mathbf{\operatorname{acc}}\\
   \theta^*(o) & o \neq\mathbf{\operatorname{acc}}
\end{cases}
\]
Since ${\btheta}$ admits arbitrary unconditional transfers, we have ${\btheta}^*_{\text{tr}} \in {\bTheta}$. Then, the following is an SPE: first, the proposing agent proposes ${\btheta}^*_{\text{tr}}$, all agents accept this contract: for other contract choices, agents accept or reject based on whether the expected value under worst-case SPE in $M+{\btheta}$ is higher than $V_i^{M, \bpi}(s_0)$ (and accepting in the case of a tie). Once contracts ${\btheta}$ are selected, in all subgames $M+{\btheta}$, agents play the worst-case SPE in $M+{\btheta}$. 
    
To argue this is an SPE, we proceed by backwards induction: once contracting is accepted or rejected, each agent plays a worst-case SPE, and moreover agents decide to accept or reject contracts appropriately according to the expected value under both. Then, it simply suffices to argue that ${\btheta}^*_{\text{tr}}$ gives the proposer maximal value amongst all choices of contract, given the responses of other agents subsequently. Upon playing ${\btheta}^*_{\text{tr}}$, agents have their value under the contract extracted exactly to the point of indifference ($V_i^{M, \bpi_0}(s_0)$), and hence accept the contract. They then play $\bpi^*$, given that $M+{\btheta}^*_{\text{tr}}$ has exactly the same SPE as $M+{\btheta}^*$, as all comparisons of value are unaffected by a constant shift of reward in all episodes. The proposing agent therefore gets payoff:
\begin{multline*}
    V_1^{M+\btheta^*, \bpi^*}(s_0) + \left(\sum_{i=2}^n V_i^{M+\btheta^*, \bpi^*}(s_0) - V_i^{M, \bpi}(s_0)\right)
    = W^{\bpi^*}(s_0) - \sum_{i=2}^n V_i^{M, \bpi}(s_0)
\end{multline*}
This is maximal: suppose for contradiction a contract giving strictly higher payoff for the proposer was available in ${\btheta'}$, which leads to equilibrium profile $\bpi'$. If the proposer proposes a contract rejected in the SPE, then no contract is implemented and agents play an SPE in $M$, giving proposer value
\[
V_1^{M, \bpi'}(s_0) = V_1^{M, \bpi}(s_0) =  W^{\bpi}(s_0) - \sum_{i=2}^n V_i^{M, \bpi}(s) \le W^{\bpi^*}(s) - \sum_{i=2}^n V_i^{M, \bpi}(s_0).
\]
Therefore, for a contract ${\btheta'}$ which is accepted, and with worst-case equilibrium $\bpi'$ played in $M+{\btheta'}$, transfers from non-proposing agent $i>1$ to the proposing agent cannot be greater than $V_i^{M+{\btheta'}, \bpi'}(s_0) - V_i^{M, \bpi}(s_0)$. Moreover, the proposing agent has a reward of (1) its own value under the game, plus (2) the time-discounted transfers through the game $M+{\btheta'}$. From this:
\begin{align*}
    V^{\mathcal{C}}((1, {\btheta'})) &= V_1^{M+{\btheta'}, \bpi'}(s_0) + \mathbb{E}_{\bpi'}\left[\sum_{t=0}^\infty \gamma^t{\btheta'}_1(o_{0,t})\right] + {\btheta'}_1(\textrm{acc}) \\
    &\le V_1^{M+\btheta', \bpi'}(s_0) + \left(\sum_{i=2}^n V_i^{M+\btheta', \bpi'}(s_0) - V_i^{M, \bpi}(s_0)\right)  \\
    &= W^{\bpi'}(s_0) - \sum_{i=2}^n V_i^{M, \bpi}(s_0) \\
    &\le W^{\bpi^*}(s_0) - \sum_{i=2}^n V_i^{M, \bpi}(s_0).
\end{align*}
This establishes that the policy profile here is indeed an SPE for the extended game $M^{\mathcal{C}}$. 

Now on to the first claim. Let $\bpi$ be any SPE for $M^{\mathcal C}$, $\bpi_0$ the equilibrium of $M$ that is played absent a contract in force. Denote $\bpi_{\btheta}$ the profile of sub-policies following the proposal of contract $\btheta$. We show that $\bpi$ must achieve at least welfare $\mathcal W (\mathcal C)$. 

Let $\varepsilon > 0$. As the contract space admits arbitrary unconditional transfers, for any choice of contract $\btheta \in \bld \Theta$, the proposer can attain, by selecting some $\btheta' \in \bTheta$ and unconditionally transferring value,
\begin{equation}
V_1^{{\mathcal{C}}}((1,{\btheta}')) = W^{M+{\btheta}, \bpi_{\btheta}}(s_0) - \sum_{i=2}^n V_i^{M, \bpi_0}(s_0) - n \varepsilon \ge \operatorname{WCSPW}_{\mathcal C} (\btheta) - \sum_{i=2}^n V_i^{M, \bpi_0}(s_0) - n \varepsilon,\label{eq:lower_bound_on_welfare}
\end{equation}
where each agent gets $\varepsilon>0$ beyond their non-contract value as an incentive to accept the contract. Under such $\btheta'$, 
\[
V_i^{{\mathcal{C}}}((1,{\btheta'})) = V_i^{M, \bpi_0}(s_0) + \varepsilon
\]
for $i=2, 3, \dots, N$. Hence, as $\varepsilon > 0$ can be arbitrarily small, substituting in and rearranging \eqref{eq:lower_bound_on_welfare}, we get 
\begin{equation}
W^{M+{\btheta'}, \bpi_{\btheta'}}(s_0) \ge \operatorname{WCSPW}_{\mathcal C} (\btheta). \label{eq:bestresponse}
\end{equation}
Now, for all $\btheta_{SPE}$ in the support of SPE $\bpi$, these must be a best response for the proposer among all choices of contract. In particular, they must be better than any choice of $\btheta'$ -- hence, substituting \eqref{eq:bestresponse}, supremizing over $\Theta$, and letting $\varepsilon \to 0$  gives
\[
W^{M+{\btheta_{SPE}}, \bpi_{\btheta_{SPE}}}(s_0) \ge \sup_{\btheta \in \bTheta} \operatorname{WCSPW}_{\mathcal C} (\btheta) = \mathcal W(\mathcal C),
\]
which proves the claim.
\end{proof}

\section{Contracting With Repeated Proposal}
Next, we prove that the statements of \Cref{thm:main} continue to hold if we have an agent $i$ that repeatedly proposes contracts. For a Markov game $M$, we say that a state $s \in S$ is unreachable from state $s_0$ is there is no sequence of action profiles $\mathbf \mathbf a_1, \mathbf \mathbf a_2, \dots, \mathbf \mathbf a_{t-1} \in \mathbf A$ such that  $s_t = s$. 
\begin{proposition}
Consider a general formulation in the sense of \Cref{apx:formulation} such that there is an agent $i$ such that $(\boldsymbol 0, j)$, $j \neq i$ is unreachable from $(\boldsymbol 0, i)$, and that $(s, \boldsymbol 0)$ is unreachable from $(s', \boldsymbol{\btheta})$ for any $s, s' \in S$, $\boldsymbol{\btheta} \in  \boldsymbol{\btheta}$. Then, the conclusions of \Cref{thm:main} continue to hold.
\end{proposition}
\begin{proof}
First, observe that the strategies in which agent $i$ repeatedly proposes a contract $\boldsymbol{\btheta}^*$ as in the proof of \Cref{thm:main} is an SPE: At each proposal and acceptance state, the agent can assume that the value is as if the agent proposes an efficient contract in the next proposal state. As in previous proofs more informally, we invoke the single-deviation principle for SPE (compare \cite{osborne1994course}) stating that to check an SPE, it is sufficient to only consider deviations by one agent at a time. First, observe that there is no reason for the proposer to propose another contract to the other agents, as the upper bound on value demonstrated in \Cref{thm:main} continues to hold. The forcing contracts still exist and incentivize all agents to follow $\bpi^*$. Finally, rejection of all contracts is again not possible as the proposing agent's best response would not be defined. Hence, the statement continues to hold for a repeatedly proposing agent.

The uniqueness of such an equilibrium follows equivalently, as an upper bound on the proposal agent is exactly attained at each proposal point.
\end{proof}
Finally, we show by means of an example that this statement breaks if multiple agents may propose a contract.
\begin{example}\label{prop:different}
Consider a repeated Prisoner's Dilemma in which agent 1 may propose as long as $(C,D)$ has been played. Then, agent $2$ proposes thereafter. Consider Prisoner's Dilemma, \Cref{subfig:pd}. An SPE of this game involves proposal of a forcing contract on $(C,D)$ with values after a signing bonus that leads to values $(-2 + \varepsilon,-1 - \varepsilon)$. All contracts that give agent $1$ more value are rejected. It is a best response for agent $2$ to accept, as long as $(-1-\varepsilon)/(1-\gamma)$, the value of accepting this contract, is larger than $-2 + 0 \times \gamma/(1-\gamma) = -2$. The best such contract is the one with the largest $\varepsilon$ for which this inequality holds.
\end{example}

\section{Contracting Augmentations with Conditional Proposals}\label{apx:formulation}
In this section, we present the augmentation for several proposing agents. Let $M = \langle S, s_0, \mathbf A, T, \mathbf R, \gamma\rangle$ be an $n$-agent Markov game, ${\btheta}$ be a contract space as before. To generalize, we require one additional component: define the \emph{contracting initiation dynamics} to be a function
\[
P \colon \{(\boldsymbol{\btheta} \cup \{\boldsymbol 0\}) \times S \times \mathbf A\} \cup \{\textrm{init}\} \to \Delta ([n] \cup \{\boldsymbol 0\} \cup \{\bullet\})
\]
which determine whether or not a new contract phase begins or ends at a given state of $M$, and if one begins, which agent is proposing the contract. More specifically, sampling from $P(\textrm{init})$ at the start of an episode, or from $P(\boldsymbol{\btheta}, s,a)$ at state-action $(s,\mathbf a)$ under contract $\boldsymbol{\btheta}$, either:

\begin{enumerate}
    \item Agent $i$ is given the option to propose a new contract, and the game is frozen ($i \in  [n]$); 
    \item None of the agents proposes a contract, but the current contract stays in force, and the game continues ($\mathbf 0$), or; 
    \item The current contract becomes void, and no agent proposes a new contract, and again the game continues ($\bullet$)
\end{enumerate}

We again assume that the contracting space ${\btheta}$ contains the null contract $\mathbf 0 (s,\mathbf a) \equiv 0$. Define the contract-augmented Markov game $M^{{\btheta}} = \langle S', S_0, \mathbf A', T', \mathbf R', \gamma\rangle$ as
 \begin{align*}
S' &= (S \cup (S \times \{\mathbf 0\})) \times ({\btheta} \cup [N])  \\
A' &= A \cup {\btheta} \cup \{\text{acc}\}\\
T'((s, {\btheta}),a) &= \begin{cases} T(s,\mathbf a) & P({\btheta}, s,\mathbf a) = \boldsymbol 0 \\
(s, \boldsymbol 0) & P(\boldsymbol{\btheta}, s, \mathbf a) = \bullet\\
(\boldsymbol 0, i) & P(\boldsymbol{\btheta}, s, \mathbf a)= i\end{cases}\\
T'(((s,\boldsymbol 0), i), (\boldsymbol{\btheta}, \mathbf a_{-i})) &= ((s,\boldsymbol 0), \boldsymbol{\btheta})\\
T'(((s,\boldsymbol 0), \boldsymbol{\btheta}), \mathbf a) &= \begin{cases}  (s, {\btheta}) & \mathbf a = \textbf{acc}\\ (s,\boldsymbol 0) & \text{else.}\end{cases}\\
R((s, {\btheta}), \mathbf a) &= R(s, \mathbf a) + {\btheta}(s,
\mathbf a).
\end{align*}

with $S_0$ being state $((s_0, \boldsymbol 0), i)$ (e.g. agent $i$ proposing a contract at the start of an episode) with probability $\mathbb{P}[P(\textrm{init})=i]$ and $(s_0, \boldsymbol 0)$ otherwise. Note that, in contrast to the definition of $S^\prime$ in \Cref{sec:augmentation}, we use tuples $(s, \boldsymbol 0)$ more generally to denote cases where state $s$ is frozen in the game while agents are negotiating a contract (which can now happen during the game, and not just at the start). For example, to capture the contracting initiation dynamics of \Cref{sec:augmentation} in this general model, we have $P(\textrm{init}) = 1$ almost surely, and $P(\cdot) = \boldsymbol 0$ almost surely for any other state (i.e., contracting is initiated and proposal rights given to agent 1 with certainty at the start of an episode, and that contract is held in force for the remainder of the episode).

More complex augmentations can be contructed in a similar way. Examples of such augmentations are multiple contracts in force at the same time, majority of acceptance as opposed to unanimity, and contract overriding only if new contracts are accepted. Exploring such augmentations is an avenue for future work.
\section{Hyperparameter Settings}\label{apx:hyperparameters}
In this section, we describe the hyperparameters we experimented with in the process of developing the study of \Cref{sec:experiments}, and emphasize the ones we ended up using in the final results in \Cref{fig:matrix}. 

\subsection{Optimization Details} In all methods, we used a learning rate of $\alpha = 10^{-4}$. Experimentation with $\alpha = 10^{-2}, 10^{-6}, 10^{-8}, 10^{-10}$, as well as a linearly decaying learning schedule from $\alpha=10^{-2}$ to $10^{-10}$, none of which improved performance on any method, the smaller learning rates significantly degrading the sample efficiency of the methods involved. 

Training happened using stochastic gradient descent, with $30$ SGD updates per training iteration and a momentum of $0.99$. For all methods second phase of MOCA, training batches consisted of ${12000}$ sampled timesteps for the static dilemmas, and ${120000}$ sampled timesteps for the dynamic games, all of them with minibatches of size ${4092}$. Other minibatch values attempted for these methods were $128, 256, {1024}$ and ${40000}$, all of which were found to degrade the performance, the first two very significantly. However, in training $\boldsymbol\bpi(s, \boldsymbol{\btheta})$, we restricted ourselves to 10\% of timesteps used in training $\boldsymbol\bpi(\boldsymbol 0, \boldsymbol 0)$ (and the baseline algorithms), and therefore found that sampling $128$ episodes per training iteration, and using a minibatch size of $128$ (in all domains) was sufficient for strong performance.

\subsection{Model Architectures} On all algorithms, we experimented with $32 \times 32$, $64 \times 64$, $256 \times 256$, $64 \times 64 \times 64$, $256 \times 256 \times 256$, and $1024 \times 1024$ MLP networks, and took the models of minimal complexity attaining maximal performance. For separate training, gifting, and contracting, we found $64 \times 64$ MLP networks to be this value, while for joint training, due to the added complexity of the joint policy, we found that expanding the network to a $256 \times 256$ MLP architecture yielded the highest performance.

\subsection{Environment Horizon and Discount Factors} All domains have discount factor $\gamma = 0.99$. Moreover, the maximum horizon for the matrix domains was $2$, for the Public Goods Dilemma was $100$, for Emergency Merge was $200$ (although episodes can terminate earlier by cars reaching the end early), and for both Harvest and Cleanup was ${1000}$. 

\subsection{PPO-specific parameters} We used a standard KL-coefficient of $0.2$,  KL-target of $0.01$, clip parameter of $0.3$, value function coefficient of $1.0$, entropy coefficient of $0.0$, for all experiments in all domains.  

\subsection{Novel hyperparameters} As mentioned in the main text, we introduced a novel hyperparameter $\boldsymbol{\nu}\boldsymbol{\nu}$, governing the number of agents sampled in the accept-reject decision for contracts. In all experiments, we used $\boldsymbol \nu  = 2$. \\

Other unmentioned hyperparameters are as found in the default settings in the RLLib implementation of PPO \cite{liang2018rllib}. 

\section{Additional Experiments}\label{apx:experiments}
In the Stag Hunt game (see \Cref{fig:staghuntgame} for a game table for two players), we observe a similar pattern as in the other studied games, outperforming separate training and gifting in the 2, 4, 8 agent cases. Additionally, in particular for larger domains, contracting outperforms joint training, due to hardness of exploration of the joint action space. Gifting does not even outperform separate training in this domain. \\

To scale this game to higher number of players, in general all agents get payout equal to the number of agents if all cooperate, and payout equal to 11 if all defect. If agents both cooperate and defect, then the defecting agents get $N - 1$ and cooperating agents get 0.\\

See \Cref{fig:staghunt} for experimental details. 
\begin{figure}[p]
    \centering
    \begin{subfigure}[b]{.32\linewidth}
    \centering
    \includegraphics[width=\linewidth]{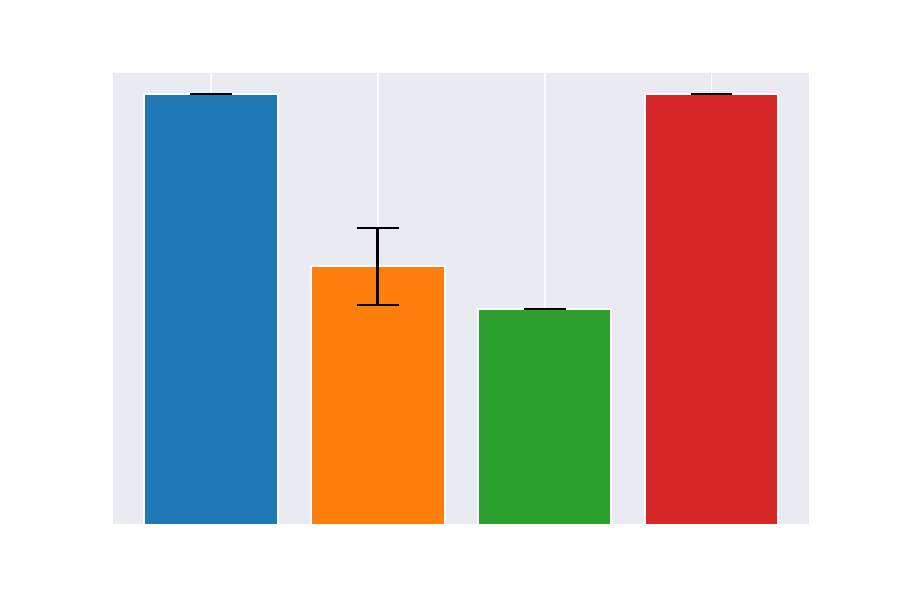}
    \caption{2 agents}
    \end{subfigure}
    \begin{subfigure}[b]{.32\linewidth}
    \centering
    \includegraphics[width=\linewidth]{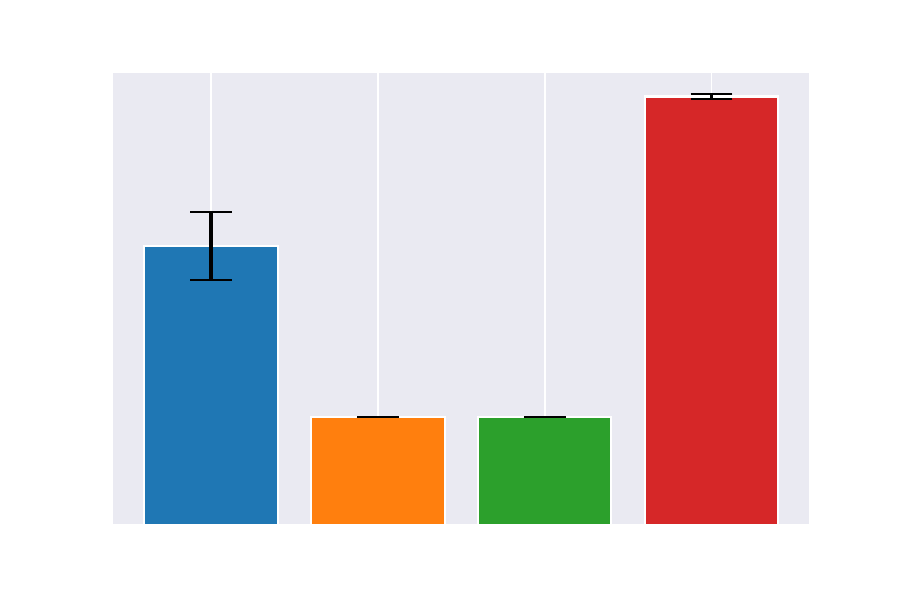}
    \caption{4 agents}
    \end{subfigure}
    \begin{subfigure}[b]{.32\linewidth}
    \centering
    \includegraphics[width=\linewidth]{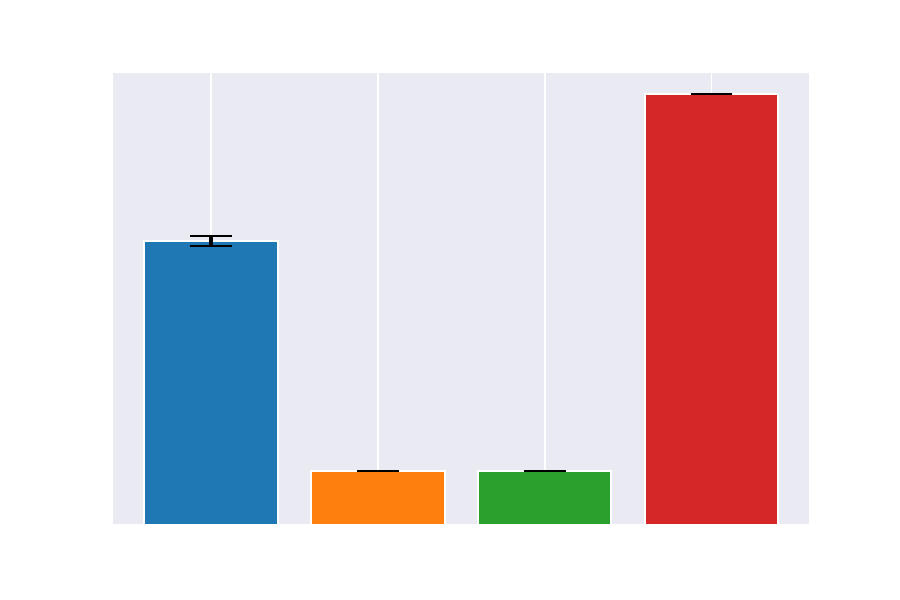}
    \caption{8 agents}
    \end{subfigure}
    \caption{Empirical Results for Stag Hunt. As in the case for Prisoner's Dilemma, contracting outperforms both separate and joint training, and always performs at least as well as joint training, significantly outperforming it for high agent counts. As in the Prisoner's Dilemma case in \Cref{fig:matrix}, 5 trials were conducted, algorithms run for 1M timesteps, and the standard error across these trials is reported in the error bars.}
    \label{fig:staghunt}
\end{figure}
\begin{figure}[p]
\centering
\begin{tabular}{|c|c|c|}
\hline 
& C & D \\
\hline 
C & 2,2 & 0,1 \\
\hline 
D & 1,0 & 1,1 \\
\hline
\end{tabular}
\caption{Stag Hunt}
\label{fig:staghuntgame}
\end{figure}

\end{document}